\documentclass{article}

\usepackage[utf8]{inputenc} 
\usepackage[margin=1.008in]{geometry}
\usepackage[T1]{fontenc}    
\usepackage{hyperref}       
\usepackage{url}            
\usepackage{booktabs}       
\usepackage{amsfonts}       
\usepackage{nicefrac}       
\usepackage{microtype}      
\usepackage{xcolor}         

\usepackage{tikz}

\title{No-Regret Learning in Two-Echelon Supply Chain \\with Unknown Demand Distribution}

\author{%
  Mengxiao Zhang \\
  University of Southern California \\  \texttt{mengxiao.zhang@usc.edu} \\ 
  \and
  Shi Chen \\
  University of Washington\\
  \texttt{shichen@uw.edu} \\
  \and
  Haipeng Luo \\
  University of Southern California \\  \texttt{haipengl@usc.edu}\\
  \and
  Yingfei Wang \\
  University of Washington\\
  \texttt{yingfei@uw.edu}
       \\
}

\newif\ifspacehack
\usepackage{natbib}
\hypersetup{
    colorlinks = blue,
    breaklinks,
    linkcolor = blue,
    citecolor = blue,
    urlcolor  = blue
}
\usepackage{url} 
\usepackage{graphicx}
\usepackage{mathtools}
\usepackage{footnote}
\usepackage{float}
\usepackage{xspace}
\usepackage{multirow}
\usepackage{xcolor}
\usepackage{wrapfig}
\usepackage{framed}
\usepackage{bbm}
\usepackage{footnote}
\usepackage{nicefrac}
\usepackage{makecell}
\usepackage[algo2e]{algorithm2e} 
\usepackage{algorithm}
\usepackage{amssymb}
\usepackage{amsthm}
\usepackage{bm}
\makesavenoteenv{tabular}
\makesavenoteenv{table}

\renewcommand{\tilde}{\widetilde}
\renewcommand{\hat}{\widehat}

\newtheorem{theorem}{Theorem}[section]

\newtheorem{lemma}[theorem]{Lemma}

\newcommand{\calA}{{\mathcal{A}}}

\newcommand{\calX}{{\mathcal{X}}}

\newcommand{\calF}{{\mathcal{F}}}
\newcommand{\calI}{{\mathcal{I}}}

\newcommand{\calD}{{\mathcal{D}}}

\newcommand{\calT}{{\mathcal{T}}}

\newcommand{\calN}{{\mathcal{N}}}

\newcommand{\Reg}{\text{\rm Reg}}

\newcommand{\contract}{\text{c}}

\newcommand{\dlow}{d}
\newcommand{\dup}{D}
\newcommand{\philow}{\gamma}
\newcommand{\phiup}{\Gamma}

\DeclareMathOperator*{\argmin}{argmin}

\newcommand{\inner}[1]{ \langle {#1} \rangle }

\newcommand{\wh}{\widehat}
\newcommand{\wt}{\widetilde}

\newcommand{\bench}{\mathring{H}}
\newcommand{\benchs}{\mathring{s}}

\newcommand{\trues}{\wh{s}}
\newcommand{\inters}{\wt{s}}

\newcommand{\order}{\ensuremath{\mathcal{O}}}
\newcommand{\otil}{\ensuremath{\tilde{\mathcal{O}}}}

\usepackage{lipsum,booktabs}
\usepackage{amsmath,mathrsfs,amssymb,amsfonts,bm,enumitem}
\usepackage{rotating}
\usepackage{pdflscape}
\usepackage{hyperref,url}
\hypersetup{
    colorlinks,
    breaklinks,
    linkcolor = blue,
    citecolor = blue,
    urlcolor  = blue,
}
\allowdisplaybreaks
\usepackage{appendix}
\usepackage{multirow,makecell,tabularx}

\usepackage{algorithmic,algorithm}

\renewcommand{\tilde}{\widetilde}
\renewcommand{\hat}{\widehat}

\def \P {\mathcal{P}}

\def \ln {\log}

\usepackage{mathtools}

\usepackage{amsthm}

\theoremstyle{definition}

\newtheorem{assumption}{Assumption}

\newtheorem{property}{Property}

\usepackage{graphicx,subfigure,color} 

\definecolor{wine_red}{RGB}{228,48,64}
\definecolor{DSgray}{cmyk}{0,1,0,0}
\newcommand{\mznote}[1]{{\Authornote{Mengxiao}{#1}}}

\usepackage{prettyref}
\newcommand{\pref}[1]{\prettyref{#1}}

\newcommand{\savehyperref}[2]{\texorpdfstring{\hyperref[#1]{#2}}{#2}}
\newrefformat{eq}{\savehyperref{#1}{Eq. \textup{(\ref*{#1})}}}
\newrefformat{eqn}{\savehyperref{#1}{Equation~\eqref{#1}}}
\newrefformat{lem}{\savehyperref{#1}{Lemma~\ref*{#1}}}
\newrefformat{def}{\savehyperref{#1}{Definition~\ref*{#1}}}
\newrefformat{line}{\savehyperref{#1}{Line~\ref*{#1}}}
\newrefformat{thm}{\savehyperref{#1}{Theorem~\ref*{#1}}}
\newrefformat{corr}{\savehyperref{#1}{Corollary~\ref*{#1}}}
\newrefformat{cor}{\savehyperref{#1}{Corollary~\ref*{#1}}}
\newrefformat{sec}{\savehyperref{#1}{Section~\ref*{#1}}}
\newrefformat{app}{\savehyperref{#1}{Appendix~\ref*{#1}}}
\newrefformat{assum}{\savehyperref{#1}{Assumption~\ref*{#1}}}
\newrefformat{ex}{\savehyperref{#1}{Example~\ref*{#1}}}
\newrefformat{fig}{\savehyperref{#1}{Figure~\ref*{#1}}}
\newrefformat{foot}{\savehyperref{#1}{Footnote~\ref*{#1}}}
\newrefformat{alg}{\savehyperref{#1}{Algorithm~\ref*{#1}}}
\newrefformat{rem}{\savehyperref{#1}{Remark~\ref*{#1}}}
\newrefformat{conj}{\savehyperref{#1}{Conjecture~\ref*{#1}}}
\newrefformat{prop}{\savehyperref{#1}{Proposition~\ref*{#1}}}
\newrefformat{proto}{\savehyperref{#1}{Protocol~\ref*{#1}}}
\newrefformat{prob}{\savehyperref{#1}{Problem~\ref*{#1}}}
\newrefformat{prot}{\savehyperref{#1}{Property~\ref*{#1}}}
\newrefformat{claim}{\savehyperref{#1}{Claim~\ref*{#1}}}
\newrefformat{que}{\savehyperref{#1}{Question~\ref*{#1}}}
\newrefformat{op}{\savehyperref{#1}{Open Problem~\ref*{#1}}}
\newrefformat{fn}{\savehyperref{#1}{Footnote~\ref*{#1}}}
\newrefformat{tab}{\savehyperref{#1}{Table~\ref*{#1}}}

\def \epsilon {\varepsilon}

\begin{document}

\maketitle

\begin{abstract}
Supply chain management (SCM) has been recognized as an important discipline with applications to many industries, where the two-echelon stochastic inventory model, involving one downstream retailer and one upstream supplier, plays a fundamental role for developing firms' SCM strategies.
In this work, we aim at designing online learning algorithms for this problem with an unknown demand distribution, which brings distinct features as compared to classic online optimization problems. Specifically, we consider the two-echelon supply chain model introduced in~\citep{MS99:Cachon_Zipkin} under two different settings: the \emph{centralized} setting, where a planner decides both agents' strategy simultaneously, and the \emph{decentralized} setting, where two agents decide their strategy independently and selfishly. We design algorithms that achieve favorable guarantees for both regret and convergence to the optimal inventory decision in both settings, and additionally for individual regret in the decentralized setting.
Our algorithms are based on Online Gradient Descent and Online Newton Step, together with several new ingredients specifically designed for our problem. We also implement our algorithms and show their empirical effectiveness.

\end{abstract}

\section{Introduction}\label{sec: intro}
A supply chain is two or more parties linked by a flow of goods, information, and funds, before a product can be finally delivered to outside customers. When multiple decision makers are involved, behavior that is locally rational can be inefficient from a global perspective. Supply chain management (SCM) research then focuses on methods for improving system efficiencies, so as to
 ``efficiently integrate suppliers, manufacturers, warehouses, and stores $\cdots$ in order to minimize system-wide costs while satisfying service level requirements''~\citep{simchi2008designing}. In the vast body of SCM literature, the mathematical model of a two-echelon stochastic inventory system with a known demand distribution plays a fundamental role for analyzing firms' SCM strategies and has been well studied over the past decades~\citep{clark1960optimal,federgruen1984computational,chen1994lower,MS99:Cachon_Zipkin}.

In the classic two-echelon stochastic inventory planning problem, two agents, Agent 1 (the retailer, referred to as \emph{he}) and Agent 2 (the supplier, referred to as \emph{she}), will go through a process of $T$ rounds. 
Following the sequence of events in the SCM literature~\citep{MS99:Cachon_Zipkin}, Agent 1 first observes an external demand $d_t\sim \calD$ and utilizes his available inventory (products in stock) to satisfy customers' demand; as a result, Agent 1 suffers either an inventory holding cost (for excess inventory) or a backorder cost (for excess demand). 
Then, Agent 1 decides his desired inventory level for round $t+1$ and orders from Agent 2. 
Next, Agent 2 handles the order from Agent 1, suffers inventory holding costs or backorder costs, decides her base-stock level for round $t+1$, and places an order from an external source (assumed to have infinite inventory). 
The two agents' orders will arrive at the beginning of the next round. The optimal policy with known demand distributions is known as the \emph{base-stock} policy for both agents \citep{clark1960optimal,federgruen1984computational,chen1994lower}.
Specifically, a base-stock policy keeps a fixed base-stock level $s$ over all time periods, meaning that if the inventory level (on-hand inventory minus the backlogged ordered) at the beginning of a period is below $s$, an order will be placed to bring the inventory level to $s$; otherwise, no order is placed. 

There are recently works extending the classic inventory control problem with known demand distribution to the one with \emph{unknown} distribution~\citep{levi2007provably,MOR2009nonparametric,huh2011adaptive,levi2015data,zhang2018perishable,chen2020optimal,MOR21:nonparametric,chen2020dynamic,SSRN2021}. However,  these works consider the single-agent case, instead of the two-echelon case.
In this work, we aim at extending the classic two-echelon stochastic inventory planning problem to an online setup with an \emph{unknown} demand distribution $\calD$. In addition, we consider the \emph{nonperishable} setting in which any leftover inventory will be carried over to the next round; as a result, the inventory level at the beginning of the next round can not be lower than the inventory level at the end of current round. The performance is measured by i) regret, the difference between their total loss and that of the best base-stock policy in hindsight; ii) last-iterate convergence to the best base-stock policy for both agents.

It is important to note that Agent 2 only observes orders from Agent 1 and does not necessarily receive the same demand information as Agent 1 does. 
In addition, in our problem formulation, Agent 1's inventory will be impacted by Agent 2's shortages. Specifically, when Agent 2 does not have enough inventory to fill Agent 1's order, we assume that Agent 2 cannot expedite to meet the shortfall, and this shortfall will cause a partial shipment to Agent 1, which implies that Agent 1 may not achieve his desired inventory level at the beginning of each round. 
This model with a known demand distribution is first examined in~\citep{MS99:Cachon_Zipkin}.   

We consider two different decision-making settings: \emph{centralized} and \emph{decentralized} settings. The centralized setting takes the perspective of a central planner who decides both agents' desired inventory level at each round in order to minimize the total loss of the entire supply chain. A more interesting and realistic setting concerns a decentralized structure in which the two agents independently decide their own desired inventory level at each round to minimize their own costs, which often results in poor performance of the supply chain (i.e., the optimal base-stock level for each agent may not be the one that achieves minimal overall loss). To achieve the optimal supply chain performance under the decentralized setting, as discussed in previous works (i.e. ~\citep{cachon2003supply}), some mechanism concerns contractual arrangement or corporate rules, such as rules for sharing the holding costs and backorder costs, accounting methods, and/or operational constraints. A \emph{contract} transfers the loss between the two agents such that each agent's objective is aligned with the supply chain's objective. However, as far as we know, this is only discussed under known demand distribution. Thus, we extend the results to the online setting with an unknown demand distribution and design learning algorithms to achieve the optimal supply chain performance.

\subsection{Techniques and Results}\label{sec: tech-results}

\textbf{Techniques.} Our problem has three salient features that are different from the classic stochastic online convex optimization problem. 
First, as will be shown in~\pref{sec: coupling},  the overall loss function is not convex with respect to both agents' inventory decisions, meaning that we can not directly apply online convex optimization algorithms to this problem. 
Second, due to the multi-echelon nature of the supply chain, Agent 2's input information is dependent on the information generated by Agent 1, which can be non-stochastic. 
Third, in the nonperishable setting, each agent's inventory level at the beginning of the next round \emph{can not} be lower than the inventory level at the end of the current round, which implies that the desired inventory level may not be always achievable. 

To address the first challenge, we introduce an augmented loss function upon which is convex and we are able to perform online convex optimization algorithms.
To address the second and the third challenge, our algorithm for both agents has the low-switching property, which only updates the strategy $\order(\log T)$ times. 
This makes Agent 2's input information almost the same as the realized demand at each round. 
For Agent 1, as he can always observe the true realized demand at each round in both centralized and decentralized setting, he makes his inventory decision based on the empirical demand distribution, which is updated $\order(\log T)$ times during the process. 
For Agent 2, in the centralized and the decentralized setting, our algorithm is a variant of Online Gradient Descent (OGD) and Online Newton Step (ONS)~\citep{ML07'log-exp-concave}, respectively. 
Both of the algorithms have the important low-switching property, which only updates the strategy $\order(\log T)$ times while at the same time achieving $\otil(\sqrt{T})$ and $\order(\log^2 T)$ regret respectively. 
We remark that our variant of ONS algorithm achieves $\order(\log^2 T)$ regret, even when the loss function is not strongly convex but satisfies a certain property. 

\textbf{Our results.} 
In the centralized setting, we design an algorithm which achieves $\otil(\sqrt{T})$ regret and last-iterate convergence to the optimal base-stock policy with rate $\otil(1/\sqrt{T})$ for Agent $1$ and $\otil(T^{-1/4})$ for Agent $2$. 
In the decentralized setting, we design a novel contract mechanism and also learning algorithms for both agents, which lead to convergence to both agents' global optimal base-stock policy with the same rate as the centralized setting. 
In addition, our algorithm guarantees that Agent $1$ has $\otil(T^{3/4})$ individual regret and Agent $2$ has $\order(\log^3 T)$ individual regret. 
Moreover, the regret with respect to the overall loss is bounded by $\otil(\sqrt{T})$, which is the same as the one in the centralized setting.~\pref{tab: summary} shows a summary of our results. 
We also implement our algorithms, and the empirical results validate the effectiveness of our algorithms (see~\pref{app: coupling-exp}). To the best of our knowledge, our work is the first one considering the two-echelon stochastic inventory planning problem in the online setup with unknown demand distribution.

\renewcommand{\arraystretch}{1.4}
\begin{table*}[t]
\caption{Summary of our results. ``Centralized'' and ``Decentralized'' represent the centralized and decentralized settings, respectively. The definitions of $\Reg_T$, $\Reg_{T,1}$ and $\Reg_{T,2}$ are introduced in~\pref{sec: coupling}. ``Convergence for Agent 1'' and ``Convergence for Agent 2'' represent the convergence rate to Agent 1 and Agent 2's optimal inventory level, respectively.\vspace{2mm}
}
\centering
\resizebox{0.95\textwidth}{!}{
\begin{tabular}{|c|c|c|c|c|c|}
\hline
Setting       & $\Reg_T$          & $\Reg_{T,1}$     & $\Reg_{T,2}$      & Convergence for Agent 1            & Convergence for Agent 2          \\ \hline
Centralized   & $\otil(\sqrt{T})$ & N/A              & N/A               & $\otil(1/\sqrt{T})$ & $\otil(T^{-1/4})$ \\ \hline
Decentralized & $\otil(\sqrt{T})$ & $\otil(T^{3/4})$ & $\otil(\log^3 T)$ & $\otil(1/\sqrt{T})$ & $\otil(T^{-1/4})$ \\ \hline
\end{tabular}}
\vspace{0mm}
\label{tab: summary}
\end{table*}

\subsection{Related Works}\label{sec: related works}
There is a vast body of SCM literature on achieving the optimal supply chain performance in the decentralized setting~\citep{lariviere1999supply,tsay1999modeling,cachon2003supply,chen2003information} concerning coordination with contract design and information sharing. 
In this body of literature, there is a line of works based on multi-echelon decentralized inventory models, which are closely related to our study, including ~\citep{lee1999decentralized, MS99:Cachon_Zipkin, lee2000value, porteus2000responsibility, watson2005decentralized, shang2009coordination}. 
However, these works all assume that the demand distribution is known (at least to the downstream agent). 

More recently, there has been growing interest in single-agent inventory control problems with unknown demand distribution~\citep{levi2007provably,MOR2009nonparametric,huh2011adaptive,levi2015data,zhang2018perishable,chen2020optimal,MOR21:nonparametric,chen2020dynamic,SSRN2021}. 
In particular,~\citep{MOR2009nonparametric} achieves $\order(\log T)$ regret in the perishable setting and $\otil(\sqrt{T})$ regret in the nonperishable setting using online gradient descent method.~\citep{SSRN2021} further extends the results to the feature-based setting.
The nonparametric approach of this line of works is fundamentally different from the conventional inventory control models in which the inventory manager knows the demand distribution (see, e.g.,~\citep{zipkin2000foundations,snyder2019fundamentals} for comprehensive reviews of the conventional inventory models); however, unlike the conventional inventory theory which has been extended from the single-echelon problems to multi-echelon problems, little has been done for the multi-echelon problems under unknown demand distributions, and we aim to fill in this gap. 

The other relevant line of works is online convex optimization.~\citep{ICML03'OCO} shows that OGD algorithm achieves $\order(\sqrt{T})$ expected regret bound for general convex functions. 
If the loss functions are exp-concave,~\citep{ML07'log-exp-concave} shows that ONS achieves $\order(\log T)$ expected regret bound. Both algorithms change their decision at every round.
On the other hand,
\citet{COLT21:lazyoco} proposes a lazy version of OGD, which changes its decision only $\order(\log T)$ times and still achieves $\otil(\sqrt{T})$ (or $\order(\log^2 T)$) regret when the loss functions are stochastically generated and convex (or strongly convex). 
In our problem, it turns out to be crucial to apply an algorithm with a small number of switches, and our algorithm generalizes the idea of~\citep{COLT21:lazyoco} to the ONS algorithm to achieve $\order(\log T)$ switches and $\otil(1)$ regret for a larger class of functions including strong convex functions. 
\section{Preliminary}\label{sec:problem_setup}

\textbf{Notations.} For a positive integer $n$, denote $[n]$ to be the set $\{1,2,\ldots,n\}$. For conciseness, we hide polynomial dependence on the problem-dependent constants in the $\order(\cdot)$ notation and only show the dependence on the horizon $T$. $\otil(\cdot)$ further hides the poly-logarithmic dependency on $T$. Define $(x)^+\triangleq \max\{x,0\}$ and $(x)^-\triangleq \max\{-x,0\}$. $\|x\|$ denotes the Euclidean norm of $x$.

Throughout this work, we make the following two mild assumptions on the demand distribution $\calD$. These mild assumptions are also made in~\citep{chen2020optimal}. 
\begin{assumption}\label{assum:bounded}
The demand distribution $\calD$ is supported on $[\dlow, \dup]$ where $\dup> \dlow>0$.
\end{assumption}
\begin{assumption}\label{assum:bounded-density}
The image of the density function of $\calD$, $\phi(\cdot)$, lies in $[\philow, \phiup]$ where $\phiup>\philow >0$.
\end{assumption}

Under the above demand assumptions, we consider the following model in the two-echelon inventory planning problem, which is first considered in~\citep{MS99:Cachon_Zipkin}. Our goal is to find the best base-stock policy. 

We first introduce the cost function under a fixed base-stock policy. 
In this model, we assume that Agent 2's inventory shortage will cause delayed (by one round) shipment and shortfalls at Agent 1 while Agent 2's orders will always be satisfied as we assume that the external source has infinite inventory. 
In addition, for unfilled demand for Agent 1, there is a backorder cost shared by the two agents, $\alpha p_1$ for Agent 1 and $(1-\alpha)p_1$ for Agent 2, where $\alpha$ is the negotiated cost sharing parameters via contractual arrangements. 
The inventory holding cost per unit for Agent 1 and Agent 2 is $h_1$ and $h_2$ respectively. 

Now we are ready to define the loss function for Agent 1 and Agent 2 respectively. Specifically, Agent 1's loss function is formulated as follows. Define
\begin{align*}
    G_1(y)=h_1\mathbb{E}_{x\sim \calD}[(y-x)^+]+\alpha p_1\mathbb{E}_{x\sim \calD}[(y-x)^-],
\end{align*}
which is Agent 1's expected sum of the holding and backorder costs per round with \emph{unlimited supply} under base-stock level $y$.
Since the actual supply to Agent 1 is limited by Agent 2's available inventory, according to~\citep{MS99:Cachon_Zipkin}, Agent 1's expected sum of the holding and backorder costs per period is defined as
\begin{align*}
H_1(s_1,s_2)\triangleq\Phi(s_2)G_1(s_1)+\int_{s_2}^{D}G_1(s_1+s_2-x)\phi(x)dx,
\end{align*}
where $\Phi(\cdot)$ is the cumulative density function of $\calD$. 
The first term is Agent 1's costs when Agent 2 has sufficient inventory to satisfy Agent 1's order (i.e., Agent 1's inventory level can be brought up to $s_1$), while the second term is the cost when Agent 2 does not have enough inventory to satisfy Agent 1's order, meaning that Agent 2's shortfall is $x-s_2$ and Agent 1's inventory can only be brought up to $s_1+s_2-x$.

For Agent 2, define 
\begin{align*}
    G_2(y)=(1-\alpha)p_1\mathbb{E}_{x\sim \calD}\left[(y-x)^-\right],
\end{align*}
which is the expected backorder cost per period incurred by Agent 2 due to Agent 1's shortages. 
Then, the expected backorder cost incurred by Agent 2 is
\begin{align*}
\Phi(s_2)G_2(s_1)+\int_{s_2}^{D}G_2(s_1+s_2-x)\phi(x)dx.
\end{align*}
The first term is the backorder cost incurred by Agent 2 due to Agent 1's shortfalls when the Agent 1's inventory level is $s_1$, while the second term is the backorder cost incurred by Agent 2 when Agent 1's inventory level is $s_1-(x-s_2)<s_1$. 
As can be seen, Agent 2's shortages $(x-s_2)$ will cause insufficient supply to Agent 1, which, in turn, will be detrimental to Agent 2 herself when Agent 1 is out of stock due to the insufficient supply. 
Therefore, Agent 2's loss function is the sum of the expected backorder cost and the expected holding cost, which is defined as 
\begin{align*}
    H_2(s_1,s_2)&\triangleq h_2\mathbb{E}_{x\sim \calD}[(s_2-x)^+]+\Phi(s_2)G_2(s_1)+\int_{s_2}^{D}G_2(s_1+s_2-x)\phi(x)dx.
\end{align*}
We also define the sum of both agents loss as $H(s_1,s_2)\triangleq H_1(s_1,s_2)+H_2(s_1,s_2)$ and $G(s)\triangleq G_1(s)+G_2(s)$.

\paragraph{Online Inventory Control} In this work, we study this conventional model in an online learning setting that proceeds in $T$ rounds.
Before the game starts, both Agent $1$ and Agent $2$ order an initial inventory level $s_{1,1}$ and $s_{1,2}$.  Then, for each round $t\in[T]$: 
\begin{itemize}
    \item at the start of round $t$, both agents' orders arrive. The current inventory level for Agent $1$ and Agent $2$ reaches to $\wh{s}_{t,1}$ and $\wh{s}_{t,2}$;
    \item external demand $d_t$ occurs at Agent 1's level where $d_t$ is drawn from the unknown demand distribution $\calD$. In this step, Agent 1 suffers from some inventory holding cost or backorder cost. Define the inventory level for Agent 1 after demand as $\inters_{t,1}$. This value can be negative as we assume backlogged orders;
    \item Agent $1$ decides his desired inventory level at the next round $s_{t+1,1}$, which leads to a demand for Agent 2: $o_{t}=(s_{t+1,1}-\inters_{t,1})^+$;
    \item Agent 2 receives the demand $o_{t}$ from Agent 1, and the inventory level after demand is $\inters_{t,2}$. Note that, in general, Agent 2 only knows $o_{t}$ instead of the real demand $d_t$. Agent 2 then suffers some inventory holding cost or backorder cost;
    \item Agent 2 decides her desired inventory level for the next round $s_{t+1,2}$ and orders $o_{t}'=(s_{t+1,2}-\inters_{t,2})^+$ from some external source.
\end{itemize}

We remark that the dynamic of the inventory for Agent 1 and Agent 2 are different. As we assume that the external source has infinite inventory, Agent 2's order can always be satisfied and we have the following dynamic for $\inters_{t,2}$ and $\wh{s}_{t+1,2}$: 
\begin{align}\label{eqn: s_2_dynamic}
    \inters_{t,2}=\wh{s}_{t,2}-o_{t}, \;\;\wh{s}_{t+1,2}=\inters_{t,2}+o_{t}'.
\end{align}
However, as Agent 2 may have delayed shipment when she does not have enough inventory, Agent 1's dynamic is defined as follows.
Define the delayed shipment of Agent 2 as $(o_{t-1}-\wh{s}_{t-1,2})^+$, which will arrive after Agent 1 has served the demand $d_t$. This means that 
\begin{align*}
    &\wt{s}_{t,1}=\wh{s}_{t,1}-d_t+(o_{t-1}-\wh{s}_{t-1,2})^+,\\
    &\wh{s}_{t+1,1}=\wt{s}_{t,1}+\min\{\wh{s}_{t,2},o_{t}\}.
\end{align*}
The specific costs suffered by the two agents in each step, as well as their objectives will be discussed in detail in~\pref{sec: decentralized-non-decoupling-model}.

\section{Main Results}\label{sec: coupling}

\subsection{Centralized Setting} \label{sec: central-planner-coupling}
In this section, we start from considering the centralized setting of our model where there is a central planner who decides both agents' strategy simultaneously. Define the loss suffered by the learner at round $t$ as follows: 
\begin{align*}
    \wt{H}_t = h_1(\trues_{t,1}-d_t)^++p_1(\trues_{t,1}-d_t)^-+h_2(\trues_{t,2}-o_t)^+,
\end{align*}
and the benchmark as the expected loss suffered by the best base-stock policy: $H(s_1^*,s_2^*)$ where $(s_1^*,s_2^*)=\argmin_{s_1,s_2}H(s_1,s_2)$. 
The regret is defined as the difference between the sum of the learners' total loss and the loss of the best base-stock policy, which is formally written as follows: 
\begin{align*}
\mathbb{E}\left[\Reg_{T}\right]=\mathbb{E}\left[\sum_{t=1}^T\left(\wt{H}_t-H(s_1^*,s_2^*)\right)\right].
\end{align*}

\begin{algorithm}[t]
   \caption{Central Planner for Coupling Model}
   \label{alg:central-lazy-three}
    \textbf{Input:} An instance of stochastic OGD $\calA$ (\pref{alg:ogd-third-party}).
    
   \textbf{Initialize:}  Arbitrary empirical cumulative density function $\wh{\Phi}_0(\cdot)$. Epoch length $L_1=1$. $\tau = 1$. 
   
   \For{$m=1,2,\ldots$}{
        \nl Define epoch $I_m=\{\tau,\tau+1,\ldots,\tau+L_m-1\}$.
    
        \nl Set $s_{m,1}=\wh{\Phi}_{m-1}^{-1}(\frac{h_2+p_1}{h_1+p_1})$. \label{line: s-1-calc-couple}
        
        \nl Receive $s_{m,2}$ from $\calA$.
        
        \While{$\tau\in I_m$}{
            \nl Decide the desired inventory level for both agents: $s_{m,1}$ for Agent $1$ and $s_{m,2}$ for Agent $2$.
            
            \nl Receive the realized demand $d_\tau$, $\tau \leftarrow \tau + 1$.
        }
        \nl Collect $\calD_m=\{d_{t'}\}_{t'\in I_m}$; define $\wh{\Phi}_{m}(x)=\frac{1}{L_m}\sum_{\tau\in I_m}\mathbb{I}\{d_\tau\leq x\}$ and also the inverse function $\wh{\Phi}_m^{-1}(z)=\min\{x:\wh{\Phi}(x)\geq z\}$; send $\wh{\Phi}_m(x)$ and $\calD_m$ to $\calA$; and set $L_{m+1}=2L_m$.
    }
\end{algorithm}

Compared with the standard online convex optimization problem~\citep{ICML03'OCO}, in which at each round the loss suffered in each round is a convex function of the current decision, our problem has two main difficulties. 
First, in our problem, the loss of the algorithm in each round depends not only on the current decided order-up-to level $s_{t,1}$ and $s_{t,2}$, but also the past decisions $s_{\tau,1}$ and $s_{\tau,2}$ for $\tau\in[t]$ as we consider the non-perishable setting, meaning that the ordered inventories can not be discarded. 
Second, even under fixed base-stock policy, the loss function $H(s_1,s_2)$ is \emph{not jointly convex} (also not convex in $s_2$).
In the following, we show how we handle these two difficulties respectively.

To deal with the first difficulty, our first key observation is that if both agents' decision $s_{t,1}$ and $s_{t,2}$ are changing very infrequently, then the loss of the algorithm at each round is almost equivalent to $\wh{H}_t(s_1,s_2)$, which is defined as:
\begin{align}\label{eqn: hhat}
    \wh{H}_t(s_{1},s_{2})\triangleq h_1(\mathring{s}_{t,1}-d_t)^++p_1(\mathring{s}_{t,1}-d_t)^-+h_2(s_{2}-d_t)^+,
\end{align}
where $\mathring{s}_{t,1}=s_{1}$ if $s_{2}>d_{t-1}$ and $\mathring{s}_{t,1}=s_{1}+s_{2}-d_{t-1}$ if $s_{2}\leq d_{t-1}$. 
Note that this loss function is a stochastic function and is only dependent on the current decision variables $(s_{1},s_{2})$.
 
To see why the loss function at round $t$ can be almost written as $\wh{H}_t(s_{t,1},s_{t,2})$ if both agents' decisions do not change very frequently, we first point out the two differences between $\wt{H}_t$ and $\wh{H}_{t}(s_1,s_2)$. First, as agents can not discard the inventories that have been ordered, the true inventory level $\wh{s}_{t,1}$ at the beginning of round $t$ may not be the desired inventory level $s_{t,1}$ when Agent $2$ does not have an inventory shortage, or $s_{t,1}+s_{t,2}-d_{t-1}$ when Agent $2$ has an inventory shortage. 
Similarly, Agent 2's true inventory level $\wh{s}_{t,2}$ may not be her desired inventory level $s_{t,2}$. Recall that $\wh{s}_{t,1}$ and $\wh{s}_{t,2}$ are used in defining $\wt{H}_t$.
Second, in $\wh{H}_{t}(s_{t,1},s_{t,2})$, the demand of Agent $2$ equals to $d_t$, while in the definition of $\wt{H}_{t}$, the demand for Agent $2$ is the order amount $o_t$ from Agent 1. 

Fortunately, these two differences can both be properly handled by a low-switching algorithm.
Specifically, suppose that both agents' desired inventory levels are kept the same: $s_{t,1}=s_1'$ and $s_{t,2}=s_2'$ for all $t$ in some time period $[t_0,t_0+L]$. Then, we can show that
\begin{align}
    &\wh{s}_{t,1}=
    \begin{cases}
        s_{1}', &\mbox{if $s_{2}'>d_{t-1}$},\\
        s_{1}'+s_{2}'-d_{t-1}, &\mbox{otherwise,}
    \end{cases}\label{eqn: low-switch-1}\\
    &\wh{s}_{t,2}=s_{2}',\label{eqn: low-switch-2}\\
    &o_t=d_t\label{eqn: low-switch-3},
\end{align}
except for at most $\Theta(1)$ rounds at the beginning of the period $[t_0, t_0+L]$
, making $\wt{H}_t = \wh{H}_t(s_{t,1},s_{t,2})$ for all the rest of the rounds. This is because~\pref{eqn: low-switch-1} does not hold only when $\wt{s}_{t-1,1}>s_{t,1}=s_1'$, which can only happen for $\Theta(1)$ rounds as the demand at each round is strictly larger than $0$ according to~\pref{assum:bounded}. Then, as $o_t=(s_{t+1,1}-\wt{s}_{t,1})^+=(s_{t+1,1}-\wh{s}_{t,1}+d_t)^+$, when $\wh{s}_{t,1}=s_1'=s_{t+1,1}$, we know that $o_t=d_t$. Similarly, as $\wh{s}_{t,2}\ne s_{t,2}$ only happens when $\wt{s}_{t-1,2}>s_{t,2}=s_2'$, and $o_t=d_t$ is strictly positive after $\Theta(1)$ rounds, we know that $\wh{s}_{t,2}=s_{t,2}=s_2'$ after another $\Theta(1)$ rounds. 
This argument is formally summarized below and proven in~\pref{app: central-planner-coupling}.

\begin{lemma}\label{lem: real-demand-coupling}
In round $t_0$, 
suppose that Agent 1 and Agent 2's desired inventory level for the following $L$ rounds is $s_1'$ and $s_2'$. Then, for some $t_1=\Theta(1)$, it holds that for all $t\in [t_0+t_1,t_0+L]$, $\wh{s}_{t,2}=s_2'$, $o_{t} = d_t$. In addition,
$\wh{s}_{t,1}=s_1'$ if $s_2'>d_{t-1}$ and $\wh{s}_{t,1}=s_1'+s_2'-d_{t-1}$ otherwise. Consequently, it holds that $\wt{H}_t=\wh{H}_t(s_1',s_2')$ during $t\in[t_0+t_1, t_0+L]$.
\end{lemma}

In addition, as proven in~\pref{lem: stationary-ot-dt} in the appendix, it holds that $\mathbb{E}[\wh{H}_t(s_1,s_2)]=H(s_1,s_2)$. This reduces our problem to optimizing over the stochastic loss $\wh{H}_t(s_1,s_2)$ with infrequent changes.

Next, we show how we handle the second difficulty, which is the issue of non-convexity of our loss function. 
Our second key observation is that with direct calculation, one can show that the optimal base-stock policy of Agent 1 has the close form: $s_1^*=\Phi^{-1}(\frac{h_2+p_1}{h_1+p_1})$ and that $H(s_1^*,s_2)$ is now convex in $s_2$; see \pref{lem: coupling-optimal} for a formal proof.
Ideally, if we set Agent 1's desired inventory level to be $s_1^*$, then we are able to apply gradient descent method to learn the best base-stock policy for Agent 2. However, we do not have the knowledge of the true demand distribution. Therefore, our solution is to construct an empirical cumulative density function $\wh{\Phi}_L(\cdot)$ for the demand distribution during the learning process where $\wh{\Phi}_L(\cdot)$ is constructed by using $L$ i.i.d. demand samples. Then, let Agent 1's desired inventory level be $s_{L,1}=\wh{\Phi}_L^{-1}(\frac{h_2+p_1}{h_1+p_1})$.

However, the expected loss function $H(s_{L,1},s_2)$ may still not be convex in $s_2$ due to the approximation error of the empirical cumulative density function. To handle this issue, we introduce the following \emph{augmented} loss function:
\begin{align}\label{eqn: aug}
    H_{L}'(s_{L,1},s_{2})\triangleq H(s_{L,1},s_{2})+(h_1+p_1)C_1\sqrt{\frac{\ln(TD/\delta)}{L}}\int_{0}^{s_2}\Phi(x)dx,
\end{align}
where $C_1>0$ is some universal constant specified in~\pref{lem: concentration lemma s-1}. 
We show in~\pref{lem: cvx-property} that $H_L'(s_{L,1},s_2)$ is indeed convex in $s_2$ with high probability.

Combining the above augmented loss function design with the idea of having a low-switching algorithm, we design our centralized algorithm~\pref{alg:central-lazy-three} as follows. The algorithm goes in epochs with exponentially increasing lengths, meaning that the number of epochs is only $\order(\log T)$. At the beginning of the $m$-th epoch $I_m$, both agents decide a fixed desired inventory level for this epoch. 
Specifically, Agent $1$ chooses his level as $s_{m,1}=\wh{\Phi}_{m-1}^{-1}(\frac{h_2+p_1}{h_1+p_1})$ (\pref{line: s-1-calc-couple}) where $\wh{\Phi}_{m-1}(\cdot)$ is the empirical cumulative density function constructed by the observed demand samples within the previous epoch $I_{m-1}$. 
Standard concentration (\pref{lem: concentration lemma s-1}) shows that $s_{m,1}$ converges to $s_1^*$. 
As discussed before, with this choice of $s_{m,1}$, the loss function $H(s_{m,1},s_2)$ may still not convex in $s_2$. According to~\pref{eqn: aug}, with a slight abuse of notation, we introduce the augmented loss function for Agent 2 at epoch $m$:
\begin{align}\label{eqn: aug-epoch-m}
    H_{m}'(s_{m,1},s_{2})\triangleq H(s_{m,1},s_{2})+(h_1+p_1)C_1\sqrt{\frac{\ln(TD/\delta)}{L_{m-1}}}\int_{0}^{s_2}\Phi(x)dx,
\end{align}
where $L_m$ is the length of epoch $I_m$ and $C_1$ is the same as the one in~\pref{eqn: aug}. 
As proven in~\pref{lem: cvx-property}, with high probability, $H_m'(s_{m,1},s_2)$ is convex in $s_2$, which enables us to apply stochastic OGD to minimize this (unknown) loss function via demands received in the previous epoch and output the average iterate as the desired inventory level for Agent 2. 
The full pseudo code is shown in~\pref{alg:ogd-third-party}.

\begin{algorithm}[t]
   \caption{Centralized Algorithm for Agent 2}
   \label{alg:ogd-third-party}
    \textbf{Input:} A set of demand value $\calD=\{d_1,\ldots,d_L\}$, empirical cumulative density function $\wh{\Phi}_L(x)=\frac{1}{L}\sum_{i=1}^L\mathbb{I}\{d_i\leq x\}$, learning rate $\eta>0$ and failure probability $\delta$.
    
   \textbf{Initialize:} 
   Set $s_{1,2}\leq D-\frac{h_2}{\phiup(h_2+p_1)}=s_{\max}$ arbitrarily.
   
   Set $s_1 = \wh{\Phi}_L^{-1}\left(\frac{h_2+p_1}{h_1+p_1}\right)$.
   
   \For{$t=1,2,\ldots,L$}{
        $s_{t+1,2} = \min\{s_{\max},\max\{0,\left(s_{t,2}-\eta\cdot m_t\right)\}\}$, where $m_t=\mathbb{I}\{s_{t,2}\leq d_{t-1}\}\left[(h_1+p_1)\mathbb{I}\{\trues_{t,1}\geq d_t\}-p_1\right]+h_2\mathbb{I}\{s_{t,2}\geq d_{t}\}+C_1(h_1+p_1)\sqrt{\frac{\ln(T\dup/\delta)}{L}}\cdot\wh{\Phi}_L(s_{t,2})$ and $\trues_{t,1}=s_1$ if $d_{t-1}\leq s_{t,2}$ and $\trues_{t,1}=s_1+s_{t,2}-d_{t-1}$ otherwise.
   }
   \Return $\bar{s}_{L,2}=\frac{1}{L}\sum_{\tau=1}^Ls_{\tau,2}$.
\end{algorithm}

This concludes our algorithm design for the centralized setting of our model. Note that as the number of epoch is $\order(\log T)$ and  both agents pick a fixed desired inventory level within each epoch, there are at most $\Theta(\log T)$ number of rounds such that~\pref{eqn: low-switch-1},~\pref{eqn: low-switch-2} and~\pref{eqn: low-switch-3} do not hold. In addition, as the epoch length $L_{m-1}$ gets longer, $H_m'(s_{m,1},s_2)$ will get closer to the true loss function $H(s_{m,1},s_2)$. Combined with the fact that $s_{m,1}$ is converging to $s_1^*$ when $m$ grows, the output of~\pref{alg:ogd-third-party}, which is the average iterate of stochastic OGD, will converge to $s_2^*$ as well. 
Moreover, it can be shown that~\pref{alg:central-lazy-three} achieves $\otil(\sqrt{T})$ regret.
See the formal statement below, the proof in~\pref{app: central-planner-coupling}, and empirical results in~\pref{app: coupling-exp}.

\begin{theorem}\label{thm: central-non-decoupling}
\pref{alg:central-lazy-three} guarantees that with probability at least $1-2\delta$, the strategy converges to the optimal base-stock policy with the following rate:
\begin{align*}
    &|s_{M,1}-s_1^*|\leq \order\left(\sqrt{\log(T/\delta)/T}\right), \\
    &|s_{M,2}-s_2^*|\leq \order\left(T^{-1/4}\log^{1/4}(T/\delta)\right),
\end{align*}
with $M=\order(\log T)$ the number of epochs. Picking $\delta=1/T^2$, \pref{alg:central-lazy-three} guarantees that $\mathbb{E}[\Reg_{T}]\leq \otil(\sqrt{T})$.
\end{theorem}

\subsection{Decentralized Setting with Contracts} \label{sec: decentralized-non-decoupling-model}
In this section, we consider how both agents learn the optimal base-stock policy $(s_1^*,s_2^*)$ in the decentralized setting where each agent decides their desired inventory level independently. As shown by~\citep{MS99:Cachon_Zipkin}, in the offline setting with \emph{known} demand distribution, to guarantee that each agent's own optimal inventory level matches the overall optimal level, a \emph{contract} is needed to reallocate the inventory holding and backorder costs between the two agents through linear payments. This contract mechanism is widely used in SCM~\citep{MS99:Cachon_Zipkin,lee1999decentralized}. Specifically, we design a contract between the two agents, which sets $\alpha=1$, meaning that Agent $1$ is responsible for all penalty costs due to his shortages, and decides a coefficient $\omega$, which is the cost that Agent $2$ needs to compensate Agent $1$ for each unsatisfied order requested by Agent $1$. 

Therefore, we define the loss suffered by Agent 1 and Agent 2 at round $t$ as follows:
\begin{align*}
    &\wt{H}_{t,1}^{\contract} \triangleq h_1(\trues_{t,1}-d_t)^++p_1(\trues_{t,1}-d_t)^--\omega_t(\trues_{t,2}-o_t)^-,\\
    &\wt{H}_{t,2}^{\contract} \triangleq h_2(\trues_{t,2}-o_t)^++\omega_t(\trues_{t,2}-o_t)^-,
\end{align*}
where $\omega_t$ is the contract coefficient agreed by both agents at round $t$. The benchmark is the loss suffered by the best base-stock policy for each agent defined as follows: 
\begin{align*}
    \bench_{t,1}^{\contract}(\benchs_1^*)&\triangleq h_1(\benchs_{t,1}^*-d_t)^++p_1(\benchs_{t,1}^*-d_t)^--\omega_t(\wh{s}_{t,2}-d_t)^-,\\
    \bench_{t,2}^{\contract}(\benchs_2^*)&\triangleq h_2(\benchs_2^*-o_t)^++\omega_t(\benchs_2^*-o_t)^-,
\end{align*}
where $\benchs_1^*=\argmin_{s_1}\mathbb{E}\left[\sum_{t=1}^T\bench_{t,1}^{\contract}(s_1)\right]$, $\benchs_2^*=\argmin_{s_2}\mathbb{E}\left[\sum_{t=1}^T\bench_{t,2}^{\contract}(s_2)\right]$ and $\benchs_{t,1}^*$ is Agent 1's inventory level at the beginning of round $t$ if he uses the base-stock policy $\benchs_1^*$. 
Note that when Agent $1$ keeps a fixed base-stock policy, it holds that $o_t=d_t$ for all $t\in[T]$.
The expected regret for each agent is defined as 
\begin{align*}
    &\mathbb{E}\left[\Reg_{T,1}\right]\triangleq\mathbb{E}\left[\sum_{t=1}^T\wt{H}_{t,1}^{\contract}-\sum_{t=1}^T\bench_{t,1}^{\contract}(\benchs_1^*)\right],\\
    &\mathbb{E}\left[\Reg_{T,2}\right]\triangleq\mathbb{E}\left[\sum_{t=1}^T\wt{H}^{\contract}_{t,2}-\sum_{t=1}^T\bench^{\contract}_{t,2}(\benchs_{2}^*)\right].
\end{align*}

Now we introduce the design of our algorithm in the decentralized setting. 
Similar to the centralized setting, in order to make sure that the loss function for each of the agent is almost only dependent on the current desired inventory level, the algorithm we design still satisfies that both agents do not update their desired inventory level very frequently, making~\pref{eqn: low-switch-1},~\pref{eqn: low-switch-2} and~\pref{eqn: low-switch-3} hold almost all the time. 
First, we introduce Agent $1$'s algorithm. 
As Agent $1$ can still observe the true demand at each round, he is able to apply the same process as shown in~\pref{alg:central-lazy-three}. Specifically, Agent $1$ still breaks the total horizon into $\order(\log T)$ epochs with exponentially increasing length and chooses his desired inventory level to be $\wh{\Phi}_{m-1}^{-1}(\frac{h_2+p_1}{h_1+p_1})$ at each epoch $I_m$ to converge to $s_1^*$.

Next, we consider the design of Agent $2$'s algorithm. 
Although Agent $2$ can also run a variant of~\pref{alg:ogd-third-party} as Agent $1$ only changes his desired level $\order(\log T)$ times, making $o_t=d_t$ except for $\order(\log T)$ rounds, Agent $2$ will suffer a $\otil(\sqrt{T})$ regret due to the approximation error of the cumulative density function. 
To achieve a better regret bound for Agent $2$, note that if Agent $2$ applies a low-switching algorithm and in addition, $\omega_t=\omega, o_t=d_t$ for all $t\in [T]$, then Agent $2$'s loss can almost be written as $\wh{H}_{t,2}^{\contract,\omega}(s)$ defined as follows:
\begin{align}\label{eqn: loss-decouple-agent-2}
    \wh{H}_{t,2}^{\contract,\omega}(s) \triangleq h_2\left(s-d_t\right)^++\omega\left(s-d_t\right)^-,
\end{align}
as we know that there are only few rounds such that $\wh{s}_{t,2}\neq s$  according to~\pref{lem: real-demand-coupling}. In addition, direct calculation shows that $\argmin_{s}\mathbb{E}_{d_t\sim \calD}[\wh{H}_{t,2}^{\contract,\omega}(s)]=\Phi^{-1}(\frac{\omega}{\omega+h_2})$.

Now, we focus on regret minimization with respect to $\wh{H}_{t,2}^{\contract}$. 
Our key observation here is that this loss function is not only convex, but also satisfy the so-called \emph{Bernstein Stochastic Gradients} property that allows faster learning.
Specifically, we prove in~\pref{lem: decouple-central-BB} that $\wh{H}_{t,2}^{\contract,\omega}$ satisfies the following property (with a specific choice of $B>0$).\footnote{We show in~\pref{lem: decouple-central-BB-discrete} that $\wh{H}_{t,2}^{^{\contract,\omega}}(s)$ satisfies \pref{prot:B-B} even when the demand is discrete.\label{foot:1}} 

\begin{property}\label{prot:B-B}
Let $\calF$ be a distribution over a class of convex functions $f:\calX\mapsto \mathbb{R}^d$. We say $\calF$ satisfies $B$-Bernstein condition with $B>0$ if for all $x\in\calX$, we have 
\begin{align*}
    (x-x^*)^\top\mathbb{E}_{f\sim \calF}\left[\nabla f(x)\nabla f(x)^\top\right](x-x^*)\leq B(x-x^*)^\top \mathbb{E}_{f\sim \calF}\left[\nabla f(x)\right],
\end{align*}
where $x^*=\argmin_{x\in \calX}\mathbb{E}_{f\sim \calF}[f(x)]$.
\end{property}

As shown by~\citet{van2016metagrad}, there exist learning algorithms that achieve $\order(\log T)$ regret bound when facing a sequence of loss functions $f_t$, each drawn independently from $\calF$ that satisfies~\pref{prot:B-B}. As a simplification, we show that the classic ONS algorithm (\pref{alg:ons} shown in~\pref{app: algorithm ons}) with a proper learning rate is already able to achieve $\order(\log T)$ regret in this case; see~\pref{thm:ONS-B-B}.

However, classic ONS changes its decision at each round. Taking inspiration from~\citep{COLT21:lazyoco}, we indeed succeed in designing a low-switching variant of ONS that changes its decision only $\otil(1)$ times while still ensuring $\otil(1)$ regret under \pref{prot:B-B}.
Specifically, our algorithm (\pref{alg:ons-lazy} shown in~\pref{app: algorithm ons}) divides the total horizon into $\order(\log T)$ epochs with exponentially increasing lengths.
While in each round it still performs the ONS update, the actual decision is only updated at the beginning of each epoch, which is set to the average of all previous ONS decisions.
It is clear that our algorithm only switches its decision $\order(\log T)$ times.
More importantly, we show that the price for the regret is only an extra $\order(\log T)$ factor, leading to an overall $\order(\log^2 T)$ regret;
see \pref{thm: low-switching-regret} for the formal statement.\footnote{In fact, one can show that a lazy version of OGD also achieves similar guarantees.
However, this heavily relies on a continuous demand distribution, while ONS works even for discrete demands (see \pref{foot:1}). Note that a discrete demand distribution is common in applications since demands usually come in a batch.} 
In additional to enjoying $\order(\log^2 T)$ regret, our algorithm in fact also ensures the last-iterate convergence to the global optimal solution of the expected loss function.
This is due to the strong convexity of $\mathbb{E}[\wh{H}^{\contract,\omega}_{t,2}(s)]$ and the fact that the last-iterate is the average of ONS updates in the previous epochs. 
We summarize the results in~\pref{lem: high-prob-concentration-lemma} and the full proof is presented in~\pref{app: lazy}. 
This means that if Agent 1 changes his desired inventory level $\order(\log T)$ times and Agent 2 applies~\pref{alg:ons-lazy} on her loss with $\omega_t=\omega$ for all $t\in[T]$, she will suffer $\order(\log^2 T)$ regret and converge to base-stock policy $\Phi^{-1}(\frac{\omega}{\omega+h_2})$. Therefore, if $\omega_t=\omega^*=\frac{h_2\Phi(s_2^*)}{1-\Phi(s_2^*)}$ for all $t\in [T]$, then applying~\pref{alg:ons-lazy}, Agent 2 will converge to $s_2^*$. 

\setcounter{AlgoLine}{0}
\begin{algorithm}[t]
   \caption{Protocol among Agent $1$, Agent $2$ and contract maker}
   \label{alg:ons-lazy-three-non-decoupling}
    \textbf{Input:} Failure probability $\delta$. Initial epoch length $L_1= 16\rho\triangleq 256(h_2+p_1)^4h_2^{-4}C_3^4\log^4(T/\delta)$, where $C_3>0$ is a constant defined in~\pref{eqn:aux-1}.
    
   \textbf{Initialize:} $s_{0,2}$ and $\wh{\Phi}_0(\cdot)$ arbitrarily. $\tau = 1$. 
   
   \For{$m=1,2,\ldots$}{
        \nl Define epoch $I_m=\{\tau,\tau+1,\ldots,\tau+L_m-1\}$.
        
        \nl Agent $1$ and Agent $2$ receive the contract coefficient $\omega_m$ from contract maker using~\pref{alg:ogd-third-party}. \label{line: contract}
        
        \nl Agent $1$ chooses $s_{m,1}=\wh{\Phi}_{m-1}^{-1}(\frac{h_2+p_1}{h_1+p_1})$. Initialize an instance $\calA$ of~\pref{alg:ons-lazy} with $\eta=\max\{\omega_m^2,h_2^2\}/(\gamma(h_2+\omega_m))$, $\epsilon=1/T$ and initial decision $s_{\tau-1,2}$.\label{line: agent-1-couple}
        
        \While{$\tau\in I_m$}{
            
            \nl Agent $1$ receives demand $d_\tau$.
            
            \nl Agent $2$ receives her ordered products from the outer resource and Agent $1$ receives his unsatisfied demand at $\tau-1$ and his inventory level goes to $\inters_{\tau,1}$.
            
            \nl Agent $1$ decides his desired inventory level at $\tau+1$ to be $s_{m,1}$, sends the order $o_\tau=\max\{s_{m,1}-\inters_{\tau,1},0\}$ to Agent $2$, and suffers loss $\wt{H}_{\tau,1}^{\contract}$. 

            \nl Agent $2$ sends the loss function $f_\tau(x)=h_2\mathbb{I}\{x\geq o_{\tau}\}+\omega_m\mathbb{I}\{x\leq o_{\tau}\}$ to $\calA$, and suffers loss $\wt{H}_{\tau,2}^{\contract}$.  \label{line:ONS1} 
            
            \nl Agent $2$ sets her own desired inventory level to be $s_{\tau,2}$, which is the output of $\calA$. \label{line:ONS2}
            
        }
        \nl Agent $1$ collects $\calD_m=\{d_{t'}\}_{t'\in I_m}$, computes $\wh{\Phi}_{m}(x)=\frac{1}{L_m}\sum_{\tau\in I_m}\mathbb{I}\{d_\tau\leq x\}$ and also the inverse function $    \wh{\Phi}_m^{-1}(z)=\min\{x:\wh{\Phi}_m(x)\geq z\}$. Set $L_{m+1}=2L_m$.
    }
\end{algorithm}

Thus, it remains to figure out how to learn $\omega^*$ for the contract maker. We design an algorithm  which updates the contract coefficient $\omega$ at the beginning of each epoch $I_m$ of Agent 1. With a slight abuse of notation, let $\omega_m$ be the contract during epoch $I_m$.
As the contract maker can observe the realized demand as well as both agents cost parameters, at the beginning of epoch $I_m$, 
we run \pref{alg:ogd-third-party} to obtain an (imaginary) inventory level $s_{m,2}'$ for Agent $2$ given $L_{m-1}$ demand samples collected during epoch $I_{m-1}$,
and then calculate $\omega_m$ following $\frac{h_2\Phi(s_2^*)}{1-\Phi(s_2^*)}$ with $s_2^*$ replaced by $s_{m,2}'$ and $\Phi$ replaced by the empirical cumulative density function. The algorithm is shown in~\pref{alg:ogd-third-party-contract} and deferred to~\pref{app: contract-maker}. The following lemma shows that given enough samples from the demand distribution, with high probability,~\pref{alg:ogd-third-party-contract} outputs a contract coefficient $\omega$ that is very close to the ideal coefficient $\omega^*$. The proof can be found in~\pref{app: beta-estimation}.
\begin{lemma}\label{lem: beta-estimation}
Let $L \geq \rho$ where $\rho$ is defined in~\pref{alg:ons-lazy-three-non-decoupling}.
Given $L$ i.i.d samples $\{d_i\}_{i=1}^L$ from the demand distribution $\calD$, with probability at least $1-\delta$, \pref{alg:ogd-third-party-contract} guarantees that i) $|\omega-\omega^*|\leq \order(L^{-{1/4}}\log(T/\delta))$, where $\omega^*=\frac{h_2\Phi(s_2^*)}{1-\Phi(s_2^*)}$; and ii) $\omega\in [0,h_2+p_1+\order(L^{-1/4}\log(T/\delta))]$.
\end{lemma}

Given this lemma, we now provide an overview of our algorithm (\pref{alg:ons-lazy-three-non-decoupling}).
It proceeds in epochs with exponentially increasing lengths again. 
At the beginning of epoch $I_m$, both agents receive a contract coefficient $\omega_m$ calculated via~\pref{alg:ogd-third-party-contract} (\pref{line: contract}). 
Then Agent 1 decides his desired inventory level to be $\wh{\Phi}_{m-1}^{-1}(\frac{h_2+p_1}{h_1+p_1})$ based on his observed demand in epoch $I_{m-1}$. 
Agent 2 then initializes an instance of~\pref{alg:ons-lazy} with the decision from the last round of the previous epoch as the initial decision (\pref{line: agent-1-couple}), and uses this instance to decide her inventory level for this epoch (\pref{line:ONS1} and \pref{line:ONS2}).
The following theorem shows that~\pref{alg:ons-lazy-three-non-decoupling} guarantees the convergence to the offline optimal solution as well as sublinear regret for both agents. 
The proof is deferred to~\pref{app: decentralized-non-decoupling}.
\begin{theorem}\label{thm: decentralized-non-decoupling}
\pref{alg:ons-lazy-three-non-decoupling} guarantees that with probability at least $1-3\delta$, 
\begin{align*}
    &|s_{M,1}-s_1^*|\leq \order\left(\sqrt{\log(T/\delta)/T}\right),\\
    &|s_{M,2}-s_2^*|\leq \order\left(T^{-{1/4}}\log(T/\delta)\right),
\end{align*}
where $M=\order(\log T)$ the total number of epochs. Picking $\delta=1/T^2$, \pref{alg:ons-lazy-three-non-decoupling} guarantees that $\mathbb{E}[\Reg_{T,1}]\leq \otil(T^{3/4})$ and $\mathbb{E}[\Reg_{T,2}]\leq \order(\log^3 T)$.
\end{theorem}

Finally, we show that~\pref{alg:ons-lazy-three-non-decoupling} also guarantees that the regret with respect to the sum of both agents' losses is bounded by $\otil(\sqrt{T})$, which is the same as the one obtained in the centralized setting. 
Note that by directly using the regret guarantees, the convergence of both agents decisions in~\pref{thm: decentralized-non-decoupling}, and the lipschitzness of the loss function, the overall regret of the loss sum can only be bounded by $\otil(T^{3/4})$. 
In order to improve the overall regret from $\otil(T^{3/4})$ to $\otil(\sqrt{T})$, we need to apply a refined analysis on Agent 2's decision sequence; see the following see~\pref{app: coupling-overall-regret} for the proof of the following theorem. 
Empirical results shown in~\pref{app: coupling-exp} also support our theoretical statements.
\begin{theorem}\label{thm: coupling-overall-regret}
    Picking $\delta= 1/T^3$, \pref{alg:ons-lazy-three-non-decoupling} guarantees that $\mathbb{E}\left[\Reg_{T}\right]\leq \otil(\sqrt{T})$.
\end{theorem}

\section{Conclusion and Future Directions}\label{sec: conclusion}
In contrast to the classic offline two-echelon stochastic inventory planning problem with known distribution studied in the SCM literature, we consider the problem with an unknown demand distribution in an online setting, which is more realistic and, as far as we know, not studied before. We consider the model formulation introduced in~\citep{MS99:Cachon_Zipkin} under both the centralized and decentralized setting, and prove both regret guarantees and convergence to the offline optimal base-stock policy.
While we assume that the true demand is observable even when it exceeds the current inventory level, a more challenging setting is the censored demand setting where 
only the amount of the satisfied demand is available.
Extending our results to the censored demand and unobserved lost sales setting appears to require new ideas.

\bibliography{ref}
\bibliographystyle{unsrtnat}

\onecolumn
\newpage
\appendix

\section{Omitted proofs in~\pref{sec: central-planner-coupling}}\label{app: central-planner-coupling}

In this section, we show the omitted proofs in the centralized setting in~\pref{sec: coupling}.

First, we prove~\pref{lem: real-demand-coupling}, which shows that if both agents keep picking the same desired inventory level $(s_{t,1}, s_{t,2})=(s_1',s_2')$ for a period of rounds, then there are at most $\Theta(1)$ rounds such that~\pref{eqn: low-switch-1},~\pref{eqn: low-switch-2} and~\pref{eqn: low-switch-3} do not hold. 
For completeness, we restate~\pref{lem: real-demand-coupling} as follows:

\begin{lemma}[Restatement of~\pref{lem: real-demand-coupling}]\label{lem: real-demand-coupling-app}
In round $t_0$, 
suppose that Agent 1 and Agent 2's desired inventory level for the following $L$ rounds is $s_1'$ and $s_2'$. Then, for some $t_1=\Theta(1)$, it holds that for all $t\in [t_0+t_1,t_0+L]$, $\wh{s}_{t,2}=s_2'$, $d_{t} = o_t$. In addition,
$\wh{s}_{t,1}=s_1'$ if $s_2'>d_{t-1}$ and $\wh{s}_{t,1}=s_1'+s_2'-d_{t-1}$ otherwise. Consequently, it hold that $\wt{H}_t=\wh{H}_t(s_1',s_2')$  for all $t\in[t_0+t_1, t_0+L]$.
\end{lemma}
\begin{proof}
     Let $\tau^*=\argmin_{\tau\in [L]}\left\{o_{t_0+\tau}>0\right\}$. Next, we first show that $\{o_{t_0+\tau}\}_{\tau=\tau^*+1}^{L-1}=\{d_{t+\tau}\}_{\tau=\tau^*+1}^{L-1}$ and $\wt{s}_{t,1}=s_1'-d_t$ for all $t\in \{\tau^*+1,\tau^*+2,\dots,L-1\}$. If $s_1'\geq \wt{s}_{t_0,1}$, then we have $o_{t_0}=s_1'-\wt{s}_{t_0,1}\geq 0$. As Agent 1 will receive the unsatisfied orders from Agent 2 before Agent 1 makes the order in the next round, at round $t+1$, Agent 1's inventory before ordering is $\wt{s}_{t_0+1}=s_1'-d_{t_0+1}$, which means that $o_{t_0+1}=s_1'-\wt{s}_{t_0+1}=d_{t_0+1}$. Repeating the above process shows that for all $t\in[t_0+1,t_0+L]$, we have $o_{t}=d_{t}$ and $\wt{s}_{t,1}=s_1'-d_{t}$.
     
     On the other hand, if $s_1'< \wt{s}_{t_0,1}$, then we have $o_{t_0}=0$ as Agent 1 can not discard the inventory. According to~\pref{assum:bounded}, we have $d_{t_0+1}\geq d$ and $\wt{s}_{t_0+1,1}\leq \wt{s}_{t_0,1}-d$. Therefore, within at most constant $t_2=\order(1)$ number of rounds, we have $\wt{s}_{t+t_2,1}\geq s_1'$. Then following the analysis in the first case proves that during $t\in[t_0+t_2,t_0+L]$, we have $o_t=d_t$ and $\wt{s}_{t,1}=s_1'-d_t$. 
     
     Next, we show that $\wh{s}_{t,2}=s_2$ after constant number of rounds. Specifically, if $\wh{s}_{t_0,2}\leq s_2'$, then at round $t_0+1$, we have $\wh{s}_{t_0+1,2}=s_2'$. Otherwise, note that when $t'\geq t_0+t_2$, $o_{t'}=d_{t'}\geq d$. Therefore, after at most constant $t_3=\order(1)$ number of rounds, we have $\wh{s}_{t_0+t_2+t_3,2}\leq s_2'$, meaning that $\wh{s}_{t,2}=s_2'$ for all $t\in [t_0+t_2+t_3+1,t_0+L]$.

     Finally, we show that $\wh{s}_{t,1}=s_1'$ if $s_2'>d_{t-1}$ and $\wh{s}_{t,1}=s_1'+s_2'-d_{t-1}$ otherwise after constant number of rounds. As shown above, when $t\geq t_0+t_2+t_3+1$, we know that $\wh{s}_{t,2}=s_2'$, $\wt{s}_{t,1}=s_1'-d_t$ and $o_t=d_t$. According to the dynamic of $\wh{s}_{t,1}$, we know that
     \begin{align*}
         \wh{s}_{t+1,1} = \wt{s}_{t,1}+\min\{\wh{s}_{t,2}, o_t\} = s_1'-d_t+\min\{s_2',d_t\}=\begin{cases}
             s_1'+s_2'-d_t &\mbox{if $s_2'<d_t$},\\
             s_1' &\mbox{otherwise.}
         \end{cases}
     \end{align*}
     
     Therefore, setting $t_1=t_2+t_3+1=\order(1)$ finishes the proof of the first statement.
     The second statement holds for $t\in[t_0+t_1,t_0+L]$ according to the definition of $\wt{H}_t$ and $\wh{H}_t(s_1',s_2')$.
\end{proof}

Next, we show that stochastic loss function defined in~\pref{eqn: hhat} is an unbiased loss estimator of the expected loss $H(s_1,s_2)$.
\begin{lemma}\label{lem: stationary-ot-dt}
$\mathbb{E}\left[\wh{H}_t(s_1, s_2)\right]=H(s_1,s_2)$, for all $t\in[T]$, where $\wh{H}_t(s_1,s_2)$ is defined in~\pref{eqn: hhat}.
\end{lemma}
\begin{proof}
According to the definition of $\wt{H}_t$, we know that
\begin{align*}
\mathbb{E}\left[\wh{H}_t(s_1,s_2)\right] &= \mathbb{E}\left[h_1(\benchs_{t,1}-d_t)^++p_1(\benchs_{t,1}-d_t)^-+h_2(s_2-d_t)^+\right] \\
    &=\mathbb{E}\left[\left(h_1(s_1-d_t)^++p_1(s_1-d_t)^-\right)\cdot\mathbb{I}\left\{s_1\in [s_1-d_{t-1},s_1-d_{t-1}+s_2]\right\}\right]+\mathbb{E}\left[h_2(s_2-d_t)^+\right].\\
    &\qquad+\mathbb{E}\left[\left(h_1(s_1+s_2-d_{t-1}-d_t)^++p_1(s_1+s_2-d_{t-1}-d_t)^-\right)\cdot\mathbb{I}\left\{s_1\geq s_1-d_{t-1}+s_2\right\}\right]\\
    &=\mathbb{E}\left[\left(h_1(s_1-d_t)^++p_1(s_1-d_t)^-\right)\cdot\mathbb{I}\left\{s_2\geq d_{t-1}\right\}\right]+\mathbb{E}\left[h_2(s_2-d_t)^+\right].\\
    &\qquad+\mathbb{E}\left[\left(h_1(s_1+s_2-d_{t-1}-d_t)^++p_1(s_1+s_2-d_{t-1}-d_t)^-\right)\cdot\mathbb{I}\left\{s_2\leq d_{t-1}\right\}\right] \\
    &=\Phi(s_2)\cdot G(s_1)+\int_{s_2}^DG(s_1+s_2-u)\phi(u)du = H(s_1,s_2).
\end{align*}
\end{proof}

The next lemma shows that the optimal solution $(s_1^*,s_2^*)$ of $H(s_1,s_2)$ satisfies that $s_1^*=\Phi^{-1}(\frac{h_2+p_1}{h_1+p_1})$ and $H(s_1^*,s_2)$ is convex in $s_2$.

\begin{lemma}\label{lem: coupling-optimal}
   Let $(s_1^*,s_2^*)=\argmin_{s_1,s_2}H(s_1,s_2)$. Then it holds that $s_1^*=\Phi^{-1}(\frac{h_2+p_1}{h_1+p_1})$ and $H(s_1^*,s_2)$ is convex in $s_2$, where $\Phi(\cdot)$ is the cumulative density function of demand distribution $\calD$.
\end{lemma}
\begin{proof}
By definition of $H(s_1,s_2)$, it holds that
\begin{align*}
    H(s_1,s_2)&=h_2\mathbb{E}_{x\sim \calD}[(s_2-x)^+]+\Phi(s_2)G(s_1)+\int_{s_2}^{D}G(s_1+s_2-u)\phi(u)du\\
    &=h_2\mathbb{E}_{x\sim \calD}[(s_2-x)^+]+\Phi(s_2)G(s_1)+\int_{0}^{D-s_2}G(s_1-u)\phi(u+s_2)du.
\end{align*}
Taking gradient over $s_1$ and $s_2$ respectively, we know that
\begin{align}
    \nabla_{s_1}H(s_1, s_2)&=(p_1+h_1)\Phi(s_2)\left(\Phi(s_1)-\frac{p_1}{p_1+h_1}\right) +(p_1+h_1) \int_{0}^{D-s_2}\left(\Phi(s_1-u)-\frac{p_1}{p_1+h_1}\right)\phi(u+s_2)du\nonumber\\
    &=(p_1+h_1)\Phi(s_2)\Phi(s_1)-p_1+(p_1+h_1)\int_{0}^{D-s_2}\Phi(s_1-u)\phi(u+s_2)du,\nonumber\\
    \nabla_{s_2}H(s_1,s_2)&=h_2\Phi(s_2)+\phi(s_2)G(s_1)+\int_{0}^{D-s_2}G(s_1-u)d\phi(u+s_2)-G(s_1+s_2-D)\phi(D)\nonumber\\
    &=h_2\Phi(s_2)+\phi(s_2)G(s_1)+\left[G(s_1-u)\phi(u+s_2)\right]_0^{D-s_2}\nonumber\\
    &\qquad +\int_{0}^{D-s_2}\phi(u+s_2)(h_1+p_1)\left(\Phi(s_1-u)-\frac{p_1}{p_1+h_1}\right)-G(s_1+s_2-D)\phi(D)\nonumber\\
    &= h_2\Phi(s_2)-p_1(1-\Phi(s_2))+(h_1+p_1)\int_{0}^{D-s_2}\Phi(s_1-u)\phi(u+s_2)du.\label{eqn: grad-s-2}
\end{align}
Setting the two gradients to be $0$, we obtain that
\begin{align*}
    &\Phi(s_2)\Phi(s_1)+\int_{0}^{D-s_2}\Phi(s_1-u)\phi(u+s_2)du=\frac{p_1}{p_1+h_1},\\
    &0=\nabla_{s_2}H(s_1,s_2)=(h_2+p_1)\Phi(s_2)-p_1-(h_1+p_1)\Phi(s_2)\Phi(s_1)+p_1=\Phi(s_2)\left[(p_1+h_2)-(h_1+p_1)\Phi(s_1)\right].
\end{align*}
Therefore, we have $s_1^*=\Phi^{-1}\left(\frac{h_2+p_1}{h_1+p_1}\right)$ and $s_2^*$ satisfies that 
\begin{align*}
    (h_2+p_1)\Phi(s_2^*)+(p_1+h_1)\int_0^{D-s_2^*}\Phi(s_1^*-u)\phi(u+s_2^*)du=p_1.
\end{align*}

Replacing $s_1$ by $s_1^*$ in~\pref{eqn: grad-s-2} and taking gradient over $s_2$, we obtain that
\begin{align}
    \nabla^2_{s_2}H(s_1^*,s_2) &= (h_2+p_1)\phi(s_2) + (h_1+p_1)\cdot \left(\int_{0}^{D-s_2}\phi(s_2+u)\phi(s_1^*-u)du-\Phi(s_1^*)\phi(s_2)\right) \label{eqn: grad-grad-s-2}\\
    &\geq \left((h_2+p_1)-(h_1+p_1)\Phi(s_1^*)\right)\phi(s_2) = 0,\nonumber
\end{align}
showing that $H(s_1^*, s_2)$ is convex in $s_2$.

\end{proof}

The next lemma follows by the standard concentration inequality, showing that the gap between $\Phi(s_{m,1})$ and $\Phi(s_1^*)$ is bounded by $\otil(L_{m-1}^{-\nicefrac{1}{2}})$ with high probability.
\begin{lemma}\label{lem: concentration lemma s-1}
Let $d_1,\dots,d_T$ be $T$ i.i.d. samples from distribution $\calD$ which satisfies~\pref{assum:bounded}. Let the empirical density distribution constructed by $\{d_i\}_{i=1}^L$ as $\wh{\Phi}_{L}(\cdot)$ defined as~\pref{eqn: empirical density function L} and the inverse of the empirical density function $\wh{\Phi}_L^{-1}(\cdot)$ as defined in~\pref{eqn: inverse-empirical-density}.
Let $s_{L,1}=\wh{\Phi}_L^{-1}(\frac{h_2+p_1}{h_1+p_1})$. Then, with probability at least $1-\delta$, for all $L\in[T]$,
\begin{align*}
    |\Phi(s_{L,1})-\Phi(s_1^*)|\leq C_1\sqrt{\frac{\ln(TD/\delta)}{L}},
\end{align*}
where $C_1>0$ is some universal constant.
\end{lemma}

\begin{proof}
Direct calculation shows that for all $L\in[T]$,
\begin{align*}
    |\Phi(s_{L,1})-\Phi(s_1^*)|\leq\phiup\left|s_{L,1}-s_1^*\right|\leq \phiup\left|\wh{\Phi}_{L}^{-1}\left(\frac{h_2+p_1}{h_1+p_1}\right)-\Phi^{-1}\left(\frac{h_2+p_1}{h_1+p_1}\right)\right|\leq C_0\sqrt{\frac{\ln(TD/\delta')}{L}},
\end{align*}
where the first inequality is by~\pref{assum:bounded-density}, the second inequality is by definition of ${s}_{L,1}$ and $s_1^*$, and the last inequality holds with probability $1-\delta$ by~\pref{lem: hoeffding}.
\end{proof}

The next lemma shows that with probability at least $1-\delta$, our constructed augmented loss function in~\pref{eqn: aug}:
\begin{align*}
    H_L'(s_{L,1},s_2)=H(s_{L,1},s_2)+(h_1+p_1)C_1\sqrt{\frac{\ln(TD/\delta)}{L}}\int_{0}^{s_2}\Phi(x)dx
\end{align*}
is convex in $s_2$ for all $L\in[T]$, where $s_{L,1}=\wh{\Phi}_L^{-1}(\frac{h_2+p_1}{h_1+p_1})$ and $C_1>0$ is defined in~\pref{lem: concentration lemma s-1}. This ensures that using (stochastic) online gradient descent with respect to $H_m'(s_{m,1},\cdot)$ defined in~\pref{eqn: aug-epoch-m} achieves sublinear regret.

\begin{lemma}\label{lem: cvx-property}
    With probability at least $1-\delta$, $H_L'(s_{L,1},s_2)\triangleq H(s_{L,1},s_2)+(h_1+p_1)C_1\sqrt{\frac{\ln(TD/\delta)}{L}}\int_{0}^{s_2}\Phi(x)dx$ is convex in $s_2$, for all $L\in[T]$. Consequently, with probability at least $1-\delta$, for all $m\in[M]$, $H_m'(s_{m,1},s_2)$ defined in~\pref{eqn: aug-epoch-m} is convex, where $M=\order(\log T)$ is the number of epochs.
\end{lemma}
\begin{proof}
According to the definition of $H_L'(s_{L,1},s_2)$, we obtain that the second-order gradient on the second parameter $s_2$ equals to
\begin{align*}
    g_{22} &=\nabla_{s_2}^2{H}_L'(s_{L,1},s_2) \\
    &=\left[(h_2+p_1)-(h_1+p_1)\Phi(s_{L,1})\right]\cdot\phi(s_2)+(h_1+p_1)\int_{s_2}^{D}\phi(s_{L,1}+s_2-u)\phi(u)du\\
    &\qquad + C_1(h_1+p_1)\sqrt{\frac{\ln(TD/\delta)}{L}}\phi(s_2)\\
    &\geq \left[(h_2+p_1)-(h_1+p_1)\left(\Phi(s_{L,1})-C_1\sqrt{\frac{\ln(TD/\delta)}{L}}\right)\right]\phi(s_2) \\
    &\geq \left[(h_2+p_1)-(h_1+p_1)\Phi(s_1^*)\right]\phi(s_2)=0,
\end{align*}
where the last inequality holds for all $L\in[T]$ with probability at least $1-\delta$ according to~\pref{lem: concentration lemma s-1}. Therefore, we know that with probability at least $1-\delta$, ${H}_L'(s_{L,1},s_2)$ is a convex function for all $L\in [T]$, meaning that $H_m'(s_{m,1},s_2)$ is also convex in $s_2$ for all $m\in[M]$ with probability at least $1-\delta$, where $M=\order(\log T)$ is the total number of epochs.
\end{proof}

In the next lemma, we show that~\pref{alg:ogd-third-party}, which applies stochastic online gradient descent with respect to $s_2$ on the augmented loss function, enjoys average-iterate convergence to the optimal solution. 

\begin{lemma}\label{lem: s-2-estimation}
Given $L$ i.i.d samples $\{d_i\}_{i=1}^L$ from the demand distribution $\calD$ with $L\geq \ln(T\dup/\delta)$. Let $s_{L,2}$ be the inventory level output by~\pref{alg:ogd-third-party} $\calA$. Then with probability at least $1-\delta$, we have $|s_{L,2}-s_2^*|\leq \order(L^{-1/4}\log^{1/4}(T\dup/\delta))$. 
\end{lemma}
\begin{proof}
According to~\pref{lem: concentration lemma s-1}, we know that with probability at least $1-\delta$, for any $L\in[T]$, 
\begin{align}\label{eqn: s-L-1-close}
    |\Phi(s_{L,1})-\Phi(s_1^*)|\leq\phiup\left|s_{L,1}-s_1^*\right|\leq \phiup\left|\wh{\Phi}_L^{-1}\left(\frac{h_2+p_1}{h_1+p_1}\right)-\Phi^{-1}\left(\frac{h_2+p_1}{h_1+p_1}\right)\right|\leq C_1\sqrt{\frac{\log(TD/\delta)}{L}},
\end{align}
Direct calculation shows that the gradient of $H_L'(s_{L,1},s_2)$ with respect to $s_2$ is as follows:
\begin{align}\label{eqn: grad-aug}
    \nabla_{s_2} {H}_{L}'(s_{L,1},s_2) = \nabla_{s_2}H(s_{L,1},s_2)+(h_1+p_1)C_1\sqrt{\frac{\log(TD/\delta)}{L}}\Phi(s_2).
\end{align}
As shown in~\pref{lem: stationary-ot-dt}, we know that
\begin{align*}
    H(s_{L,1},s_2) = \mathbb{E}\left[h_1(\wh{s}_{L,1}-d_t)^++p_1(\wh{s}_{L,1}-d_t)^-+h_2(s_2-d_{t})^+\right],
\end{align*}
where $\wh{s}_{L,1}=s_{L,1}$ if $s_2\geq d_{t-1}$ and $\wh{s}_{L,1}={s}_{L,1}+s_2-d_{t-1}$ otherwise. Therefore, an unbiased estimator of the first term of the right hand side of ~\pref{eqn: grad-aug} is as follows:
\begin{align*}
    \mathbb{I}\{s_2\leq d_{t-1}\}\left[(h_1+p_1)\mathbb{I}\{\wh{s}_{L,1}\geq d_t\}-p_1\right]+h_2\mathbb{I}\{s_{2}\geq d_{t}\}.
\end{align*}
For the second term, as $\wh{\Phi}_L(\cdot)$ is an unbiased estimator of $\Phi(\cdot)$, we know that
\begin{align*}
    (h_1+p_1)C_1\sqrt{\frac{\log(TD/\delta)}{L}}\Phi(s_2) = (h_1+p_1)C_1\sqrt{\frac{\log(TD/\delta)}{L}}\mathbb{E}\left[\wh{\Phi}_L(s_2)\right].
\end{align*}
Therefore, we can indeed construct an unbiased estimator of the true gradient $\nabla_{s_2}{H}_{L}'(\wh{s}_{L,1},s_2)$ and run stochastic online gradient descent. Specifically, as shown in~\pref{alg:ogd-third-party}, let
\begin{align*}
    m_t=\mathbb{I}\{s_{t,2}\leq d_{t-1}\}\left[(h_1+p_1)\mathbb{I}\{\wh{s}_{L,1}\geq d_t\}-p_1\right]+h_2\mathbb{I}\{s_{t,2}\geq d_{t}\}+C_1(h_1+p_1)\sqrt{\frac{\log(TD/\delta)}{L}}\wh{\Phi}_L(s_{t,2}).
\end{align*}
Then based on the above calculation, we know that $\mathbb{E}[m_t]=\nabla_{s_2}H_L'(s_{L,1},s_{t,2})$.
Moreover, as $L\geq \log(TD/\delta)$, we know that $|m_t|\leq \max\{h_1,p_1\}+C_1(h_1+p_1)=\order(1)$. According to classic online gradient descent analysis (e.g. Theorem 3.1.1 in~\citep{hazan2016introductionOCO}), \pref{alg:ogd-third-party} guarantees that for any $s_2\leq D-\frac{h_2}{\phiup (h_2+p_1)}$,
\begin{align}\label{eqn:ogd-regret-1}
    \mathbb{E}\left[\sum_{\tau=1}^L {H}_{L}'(s_{L,1},s_{\tau,2})-\sum_{\tau=1}^L{H}_L'(s_{L,1},s_2)\right]\leq \order(\sqrt{L}),
\end{align}
where we omit all the problem-dependent constants here. 

Next, we show that $s_2^*\leq D-\frac{h_2}{\phiup(h_2+p_1)}$. From the optimality condition of $s_1^*$ and $s_2^*$ and~\pref{eqn: grad-s-2}, we know that
\begin{align*}
    s_1^* &= \Phi^{-1}\left(\frac{h_2+p_1}{h_1+p_1}\right),\\
    0=\nabla_{s_2} H(s_1^*,s_2^*)&= (h_2+p_1)\Phi(s_2^*)-p_1+(p_1+h_1)\int_0^{D-s_2^*}\Phi(s_1^*-u)\phi(u+s_2^*)du \\
    &\geq (h_2+p_1)\Phi(s_2^*)-p_1.
\end{align*}
Therefore, it holds that
\begin{align}\label{eqn: s-2-star-bound}
    s_2^*\leq \Phi^{-1}\left(\frac{p_1}{h_2+p_1}\right).
\end{align}
Furthermore, as $\phi(x)\in [\philow, \phiup]$ for all $x\in[\dlow,\dup]$, 
\begin{align*}
    \phiup (D-s_2^*)\geq \Phi(D)-\Phi(s_2^*)\geq 1-\frac{p_1}{h_2+p_1} = \frac{h_2}{h_2+p_1} \Longrightarrow\; s_2^*\leq D-\frac{h_2}{\phiup(h_2+p_1)}.
\end{align*}

Therefore, according to the $\max\{h_1,p_1\}$-Lipschitzness of ${H}_L'(s_1,s_2)$ in $s_1$, we have
\begin{align*}
    &\mathbb{E}\left[\sum_{\tau=1}^L{H}_L'(s_1^*,s_{\tau,2}) - \sum_{\tau=1}^L{H}_L'(s_1^*,s_{2})\right]\\
    &\leq \mathbb{E}\left[\sum_{\tau=1}^L{H}_L'(s_{L,1},s_{\tau,2}) - \sum_{\tau=1}^L{H}_L'(s_{L,1},s_{2})\right] +\order\left(\sqrt{L\log(T\dup/\delta)}\right)\tag{Lipschitzness of $H_L'(s_1,s_2)$ and~\pref{eqn: s-L-1-close}}\\
    &\leq \order(\sqrt{L\log(T\dup/\delta)}) \tag{by~\pref{eqn:ogd-regret-1}}.
\end{align*}
Choosing $s_2=s_2^*$, $\delta=\frac{1}{T^2}$ and using the definition of ${H}_L'(s_1,s_2)$, we can obtain that
\begin{align}
    &\mathbb{E}\left[\sum_{\tau=1}^LH(s_1^*, s_{\tau,2})-\sum_{t=1}^LH(s_1^*,s_2^*)\right] \nonumber\\
    &\leq \mathbb{E}\left[\sum_{\tau=1}^L{H}_L'(s_1^*,s_{\tau,2})-\sum_{\tau=1}^L{H}_L'(s_1^*,s_2^*)\right] + 2C_1L(h_1+p_1)\mu\cdot\sqrt{\frac{\log(TD/\delta)}{L}} +\order(1)\nonumber \\
    &\leq \mathbb{E}\left[\sum_{\tau=1}^L{H}_L'(s_1^*,s_{\tau,2})-\sum_{\tau=1}^L{H}_L'(s_1^*,s_2^*)\right] + 2C_1(h_1+p_1)\mu\sqrt{L\log(TD/\delta)} +\order(1)\nonumber\\
    &\leq \order(\sqrt{L\log T}) \label{eqn: reg_central},
\end{align}
where $\mu=\mathbb{E}_{x\sim \calD}[x]\leq \dup$ and $\order(\cdot)$ hides all problem-dependent constants.

Moreover, note that $H(s_1^*,s_2)$ is $\sigma_2''$-strongly convex in $s_2\leq D-\frac{h_2}{\phiup(h_2+p_1)}$ as according to~\pref{eqn: grad-grad-s-2},
\begin{align*}
    \nabla^2H(s_1^*,s_2)&=\left[(h_2+p_1)-(h_1+p_1)\Phi(s_1^*)\right]\cdot\phi(s_2)+(h_1+p_1)\int_{s_2}^D\phi(s_1^*+s_2-u)\phi(u)du\\
    &\geq(h_1+p_1)(D-s_2)\philow^2 \\
    &\geq\frac{\philow^2(h_1+p_1)}{\phiup(h_2+p_1)}\triangleq \sigma_2''.
\end{align*}
Therefore, according to~\pref{lem: high-prob-concentration-lemma}, we know that with probability at least $1-\delta$,
\begin{align}\label{eqn: s-2-convergence}
    |\bar{s}_{L,2}-s_2^*|\leq \order\left(\sqrt{\frac{\sqrt{L\log(T\dup/\delta)}}{L}}+\sqrt{\frac{\log(1/\delta)}{L}}\right) = \order\left(L^{-\frac{1}{4}}\log^{\frac{1}{4}}(T\dup/\delta)\right),
\end{align}
which finishes the proof.

\end{proof}

Now we are ready to prove our main result~\pref{thm: central-non-decoupling} in the central planner setting. For completeness, we restate the theorem as follows.

\begin{theorem}[Restatement of~\pref{thm: central-non-decoupling}]\label{thm: central-non-decoupling-app}
\pref{alg:central-lazy-three} guarantees that with probability at least $1-2\delta$, the strategy converges to the optimal base-stock policy with the following rate:
\begin{align*}
    &|s_{M,1}-s_1^*|\leq \order\left(\sqrt{\log(T/\delta)/T}\right), \\
    &|s_{M,2}-s_2^*|\leq \order\left(T^{-1/4}\log^{1/4}(T/\delta)\right),
\end{align*}
with $M=\order(\log T)$ the total number of epochs. Picking $\delta=1/T^2$, \pref{alg:central-lazy-three} also guarantees that $\mathbb{E}[\Reg_{T}]\leq \otil(\sqrt{T})$.
\end{theorem}

\begin{proof}
We first prove the convergence of $s_{M,1}$ and $s_{M,2}$. According to~\pref{eqn: s-L-1-close} and~\pref{lem: s-2-estimation}, we know that with probability at least $1-2\delta$, for all $m\in[M]$,
\begin{align}
    &\left|s_{m,1}-s_1^*\right|\leq \frac{C_1}{\phiup}\sqrt{\frac{\log(TD/\delta)}{L_{m-1}}}=\order\left(\sqrt{\frac{\log(T/\delta)}{2^m}}\right),\label{eqn: s-m-1-converge}\\
    &\left|s_{m,2}-s_2^*\right|\leq \order\left(L_{m-1}^{-\frac{1}{4}}\log^{\frac{1}{4}}(T/\delta)\right)=\order\left(2^{-\frac{m}{4}}\log^{\frac{1}{4}}(T/\delta)\right).\label{eqn: s-m-2-converge}
\end{align}
Picking $m=M$ and noticing the fact that $2^m=\Theta(T)$ prove the convergence of $s_{M,1}$ and $s_{M,2}$.

According to~\pref{eqn: reg_central}, picking $\delta=\frac{1}{T^2}$, we obtain that
\begin{align*}
    &L_m\cdot \mathbb{E}\left[H(s_1^*,s_{m,2})-H(s_1^*,s_2^*)\right]\\
    &=\mathbb{E}\left[\sum_{t\in I_m}H(s_1^*,s_{m,2})-\sum_{t\in I_m}H(s_1^*,s_2^*)\right]\leq \order\left(\sqrt{L_m\log T}\right).
\end{align*}
Furthermore, according to the $\max\{h_1,p_1\}$-Lipschitzness of $H(\cdot,s_2)$ for any $s_2$ and~\pref{eqn: s-m-1-converge}, we know that
\begin{align}
    &L_m\cdot\mathbb{E}\left[H(s_{m,1},s_{m,2})-H(s_1^*,s_2^*)\right] \nonumber\\
    &\leq L_m\cdot\mathbb{E}\left[H(s_{1}^*,s_{m,2})-H(s_1^*,s_2^*)\right] + L_m\order\left(|s_{m,1}-s_1^*|\right)\nonumber\\
    &\leq \order(\sqrt{L_m\log T}).\label{eqn: aux-cou-1}
\end{align}
To further show that the expected regret is also well-bounded, as proven in~\pref{lem: real-demand-coupling}, within each epoch, there is only constant number of rounds such that $\wt{H}_t\neq \wh{H}_t(s_{m,1},s_{m,2})$, $t\in I_m$. Therefore, picking $\delta = 1/T^2$, we have,
\begin{align*}
    \mathbb{E}\left[\Reg_{T}\right] &= \mathbb{E}\left[\sum_{t=1}^T\wt{H}_{t}-\sum_{t=1}^T{H}(s_1^*,s_2^*)\right]\\
    &=\mathbb{E}\left[\sum_{m=1}^M\sum_{t\in I_m}\left(\wt{H}_{t}-H(s_{m,1},s_{m,2})\right)\right]+\mathbb{E}\left[\sum_{m=1}^M\left(\sum_{t\in I_m}H(s_{m,1},s_{m,2})-\sum_{t\in I_m}H(s_1^*,s_2^*)\right)\right]\\
    &=\mathbb{E}\left[\sum_{m=1}^M\sum_{t\in I_m}\left(\wt{H}_{t}-\wh{H}_t(s_{m,1},s_{m,2})\right)\right]+\mathbb{E}\left[\sum_{m=1}^M\left(\sum_{t\in I_m}H(s_{m,1},s_{m,2})-\sum_{t\in I_m}H(s_1^*,s_2^*)\right)\right]\\
    &\leq\order(\log T)+\mathbb{E}\left[\sum_{m=1}^M\order\left(\sqrt{L_m\log(T\dup/\delta)}\right)\right]\tag{\pref{eqn: aux-cou-1}}\\
    &\leq \order(\sqrt{T\log T}),
\end{align*}
where the last inequality is because of the exponential length scheduling of the epochs.
\end{proof}
\section{Omitted proofs in~\pref{sec: decentralized-non-decoupling-model}}\label{app: decentralized-non-decoupling-model}

\subsection{ONS and Lazy ONS algorithm}\label{app: algorithm ons}
We show full pseudo code of the classic ONS algorithm in~\pref{alg:ons} and our proposed lazy ONS algorithm in~\pref{alg:ons-lazy}.

\setcounter{AlgoLine}{0}
\begin{algorithm}[t]
   \caption{Online Newton Step}
   \label{alg:ons}
    \textbf{Input:} learning rate $\eta>0$, perturbation $\epsilon>0$.
    
   \textbf{Initialize:} $x_1=x_0$ arbitrarily.
   
   \For{$t=1$ to $T$}{
        Choose action $x_t\in \calX$ and observe $g_t=\nabla f_t(x_t)$.
        
        Update $x_{t+1}=\Pi_{\calX}^{M_t}(x_t-\eta M_t^{-1}g_t)$, where $M_t=\sum_{s=1}^tg_sg_s^\top +\epsilon I$.
    }
\end{algorithm}

\setcounter{AlgoLine}{0}
\begin{algorithm}[t]   \caption{Online Newton Step with lazy update}
   \label{alg:ons-lazy}
    \textbf{Input:} learning rate $\eta$, perturbation $\epsilon>0$, total horizon $T$.
    
   \textbf{Initialize:} $\wh{x}_1=x_1=x_0$ arbitrarily, $k=0$.
   
   \For{$t=1$ to $T$}{
        
        \nl \If {$t=2^k$}{
         \nl    $k\leftarrow k+1$
        
        \nl $\wh{x}_k=\frac{1}{t}\sum_{s=1}^tx_t\in\calX$.
        }
        \nl Choose action $w_t=\wh{x}_k\in \calX$ and observe $f_t$. 

        \nl Set $g_t=\nabla f_t(x_t)$
        
        \nl Update $x_{t+1}=\Pi_{\calX}^{M_t}(x_t-\eta M_t^{-1}g_t)$, where $M_t=\sum_{s=1}^tg_sg_s^\top +\epsilon I$. \label{line: ONS-per-step-update}
    }
\end{algorithm}

\subsection{$\wh{H}_{t,2}^{\contract,\omega}(x)$ satisfies \pref{prot:B-B}}\label{app: decouple-central-BB}
\begin{lemma}\label{lem: decouple-central-BB}
Suppose that demand distribution $\calD$ satisfies and~\pref{assum:bounded-density}. The stochastic function $\wh{H}_{t,2}^{\contract,\omega}$ defined in~\pref{eqn: loss-decouple-agent-2} satisfies~\pref{prot:B-B} with $B=\frac{\max\{\omega^2,h_2^2\}}{\philow(h_2+\omega)}$ for all $t\in [T]$.
\end{lemma}
\begin{proof}
Direct calculation shows that $$H_2^{\contract,\omega}(x)\triangleq\mathbb{E}\left[\wh{H}_{t,2}^{\contract,\omega}(x)\right]=(h_2+\omega)\mu+(h_2+\omega)\int_{0}^x\left(\Phi(u)-\frac{\omega}{h_2+\omega}\right)du,$$
    where $\mu=\mathbb{E}_{d'\sim \calD}[d']$. Also it is direct to see that the minimizer of $H_2^{\contract,\omega}(x)$ is $x^*=\Phi^{-1}\left(\frac{\omega}{\omega+h_2}\right)$. Taking the gradient of $\wh{H}_{t,2}^{\contract,\omega}(x)$, we have:
    \begin{align*}
        \nabla \wh{H}_{t,2}^{\contract,\omega}(x)&=(h_2+\omega)\mathbb{I}\{x\geq d_t\}-\omega,\\
        \nabla \wh{H}_{t,2}^{\contract,\omega}(x)\cdot\nabla\wh{H}_{t,2}^{\contract,\omega}(x) &= (h_2+\omega)^2\mathbb{I}\{x\geq d_t\}-2\omega(h_2+\omega)\mathbb{I}\{x\geq d_t\}+\omega^2.
    \end{align*}
    Taking expectation of the above two equations, we have
    \begin{align*}
        \mathbb{E}\left[\nabla \wh{H}_{t,2}^{\contract,\omega}(x)\right]&=(h_2+\omega)\Phi(x)-\omega,\\
        \mathbb{E}\left[\nabla \wh{H}_{t,2}^{\contract,\omega}(x)\cdot \nabla \wh{H}_{t,2}^{\contract,\omega}(x)\right] &= \omega^2+\Phi(x)(h_2^2-\omega^2).
    \end{align*}
    To show that $\wh{H}_{t,2}^{\contract,\omega}(x)$ satisfies~\pref{prot:B-B}, we first consider the case $x\geq x^*=\Phi^{-1}\left(\frac{\omega}{\omega+h_2}\right)$. In this case, we need to find $B>0$ such that for all $x\geq x^*$:
    \begin{align*}
        B\geq \frac{(x-x^*)(\omega^2+\Phi(x)(h_2^2-\omega^2))}{(\omega+h_2)\Phi(x)-\omega}.
    \end{align*}
    
    Using~\pref{assum:bounded-density}, we have $(x-x^*)\leq \frac{1}{\philow(h_2+\omega)}((h_2+\omega)\Phi(x)-\omega)$, which means that
    \begin{align*}
        \frac{(x-x^*)(\omega^2+\Phi(x)(h_2^2-\omega^2))}{(\omega+h_2)\Phi(x)-\omega}\leq \frac{\omega^2+\Phi(x)(h_2^2-\omega^2)}{\philow(h_2+\omega)}\leq \frac{\max(\omega^2,h_2^2)}{\philow(h_2+\omega)}.
    \end{align*}
    Choosing $B\geq\frac{\max(\omega^2,h_2^2)}{\philow(h_2+\omega)}$ satisfies~\pref{prot:B-B}. The second case where $x\leq x^*$ can be proved in a similar way. Therefore, we show that $\wh{H}_{t,2}^{\contract,\omega}(x)$ satisfies~\pref{prot:B-B}.
\end{proof}

As claimed in~\pref{foot:1}, we show in the following lemma that even when the demand distribution is discrete and the expected loss function is \emph{not} strongly convex, the realized loss function $\wh{H}_{t,2}^{\contract,\omega}(x)$ also satisfies~\pref{prot:B-B}. 
\begin{lemma}\label{lem: decouple-central-BB-discrete}
Suppose that demand distribution is supported on finite values $d_i>0$ with probability $w_i>0$, $i\in [k]$, $\sum_{i=1}^kw_i=1$ and $d_1<d_2<\ldots<d_k$. Also suppose that there exists a unique $i^*\in [k]$ such that $\Phi(d_{i^*-1})<\frac{\omega}{h_2+\omega}$ and $\Phi(d_{i^*})>\frac{\omega}{h_2+\omega}$. Let $\theta=\min\{\Phi(d_{i^*})-\frac{\omega}{h_2+\omega}, \frac{\omega}{h_2+\omega}-\Phi(d_{i^*-1})\}$. The stochastic function $\wh{H}_{t,2}^{\contract,\omega}(x)$ defined in~\pref{eqn: loss-decouple-agent-2} satisfies~\pref{prot:B-B} with $B=\frac{\max_{i\in [k]}{d_i}\cdot\max\{\omega^2,h_2^2\}}{\theta(h_2+\omega)}$.
\end{lemma}
\begin{proof}
    We first show that $\mathbb{E}_{d_t\sim \calD}[\wh{H}_{t,2}^{\contract,\omega}(x)]$ is not strongly convex. In fact, direct calculation shows that
    \begin{align*}
        &\mathbb{E}_{d_t\sim \calD}\left[\wh{H}_{t,2}^{\contract,\omega}(x)\right] =\sum_{i=1}^kw_i\left(\omega(x-d_i)^++h_2(x-d_i)^-\right),
    \end{align*}
    which is a piece-wise linear function, thus not strongly convex.
    
    To show that $\wh{H}_{t,2}^{\contract,\omega}(x)$ satisfies~\pref{prot:B-B}, direct calculation shows that
    \begin{align*}
    \mathbb{E}\left[\nabla \wh{H}_{t,1}^{\contract,\omega}(x)\right]&=(h_2+\omega)\Phi(x)-\omega,\\
    \mathbb{E}\left[\nabla \wh{H}_{t,2}^{\contract,\omega}(x)\cdot \nabla \wh{H}_{t,2}^{\contract,\omega}(x)\right] &= \omega^2+\Phi(x)(h_2^2-\omega^2).
    \end{align*}
    It is also direct to see that the minimizer of $\mathbb{E}[\wh{H}_{t,2}^{\contract,\omega}(x)]$ is $x^*=d_{i^*}$. When $x\geq x^*$, we need to show that there exists $B>0$ such that for all $x>x^*$,
    \begin{align}\label{eqn: cond-1-discrete}
        B\geq \frac{(x-x^*)(\omega^2+\Phi(x)(h_2^2-\omega^2))}{(\omega+h_2)\Phi(x)-\omega}.
    \end{align}
    Note that for $x\geq x^*$, $\Phi(x)\geq \Phi(x^*)\geq \theta+\frac{\omega}{\omega+h_2}$ and $x-x^*\leq \max_{i\in [k]}d_i$, therefore, we have
    \begin{align*}
        \frac{(x-x^*)(\omega^2+\Phi(x)(h_2^2-\omega^2))}{(\omega+h_2)\Phi(x)-\omega}\leq \frac{\max_{i\in [k]}d_i\cdot \max\{\omega^2,h_2^2\}}{\theta(\omega+h_2)},
    \end{align*}
    meaning that $B=\frac{\max_{i\in [k]}d_i\cdot \max\{\omega^2,h_2^2\}}{\theta(\omega+h_2)}$ satisfies~\pref{eqn: cond-1-discrete}.
    Similarly, when $x\leq x^*$, we can also show that $B=\frac{\max_{i\in [k]}d_i\cdot \max\{\omega^2,h_2^2\}}{\theta(\omega+h_2)}$ satisfies that
    \begin{align*}
        B\geq \frac{(x^*-x)(\omega^2+\Phi(x)(h_2^2-\omega^2))}{\omega-(\omega+h_2)\Phi(x)}.
    \end{align*}
    Combining both cases shows that $\wh{H}_{t,2}^{\contract,\omega}(x)$ satisfies~\pref{prot:B-B}.
\end{proof}

\subsection{ONS achieves $\order(\log T)$ regret when~\pref{prot:B-B} is satisfied}\label{app:ONS-B-B}

\begin{theorem}\label{thm:ONS-B-B}
Let $\calX\subseteq\mathbb{R}^d$ be a convex set with bounded diameter $\max_{x,x'\in\calX}\|x-x'\|\leq J$. If $\{f_t\}_{t=1}^T$ satisfy~\pref{prot:B-B} for some $B>0$, $f_t:\calX\mapsto \mathbb{R}$ and $\max\|\nabla f_t(x)\|\leq G$, \pref{alg:ons} with $\eta\geq 2B$ and $\epsilon=1/T$ ensures: $\mathbb{E}[\sum_{t=1}^Tf_t(x_t)-\sum_{t=1}^Tf_t(x^*)]\leq \order(d\log (GT)+J^2/2BT)$, where $x^*=\argmin_{x\in\calX}\mathbb{E}[f_t(x)]$.
\end{theorem}

\begin{proof} The first part of the proof follows the classic ONS proof: let $y_{t+1}=x_t-\eta M_t^{-1}g_t$ and we know that
    \begin{align*}
        y_{t+1}-x^*&=x_t-x^*-\eta M_t^{-1}g_t,\\
        M_t(y_{t+1}-x^*)&=M_t(x_t-x^*)-\eta g_t.
    \end{align*}
    
    Therefore, by definition of $x_{t+1}$, we know that
    \begin{align*}
        \|x_{t+1}-x^*\|_{M_t}^2\leq \|y_{t+1}-x^*\|_{M_t}^2 = \|x_t-x^*\|_{M_t}^2-2\eta \inner{x_t-x^*, g_t}+\eta^2\|g_t\|_{M_t^{-1}}^{2},
    \end{align*}
    where $\|x\|_M^2\triangleq x^\top Mx$. Rearranging the terms, we know that
    \begin{align*}
        \inner{x_t-x^*, g_t}\leq \frac{\|x_{t}-x^*\|_{M_t}^2-\|x_{t+1}-x^*\|_{M_t}^2}{\eta}+\eta\|g_t\|_{M_t^{-1}}^2.
    \end{align*}
    
    Taking summation over $t\in [T]$ using the definition of $M_t$, we know that
    \begin{align*}
        \sum_{t=1}^T\inner{x_t-x^*,g_t}\leq \frac{\|x_1-x^*\|_{M_0}^2}{\eta}+\frac{1}{\eta}\sum_{t=1}^T(x_t-x^*)^\top g_tg_t^\top(x_t-x^*)+ \eta\sum_{t=1}^T\|g_t\|_{M_t^*}^2.
    \end{align*}
    
    By choosing $\epsilon=\frac{1}{T}$, we have the first term bounded by $\order(\frac{J^2}{\eta T})$. For the third term, according to the assumption that $\|g_t\|_2\leq G$ and Lemma 6 in~\citep{ML07'log-exp-concave}, we obtain that
    \begin{align*}
        \sum_{t=1}^T\|g_t\|_{M_t^*}^2\leq \log\left(\frac{\det(\sum_{t=1}^Tg_tg_t^\top+\epsilon I)}{\det(\epsilon I)}\right)\leq d\log\left(\frac{G^2T}{\epsilon}+1\right)\leq 4d\log(GT).
    \end{align*}
    
    Finally, we consider the second term. Let $f(x)=\mathbb{E}_{f_t\sim \calF}[f_t(x)]$. According to the convexity of $f$, we have for any $x,y\in \calX$,
    \begin{align*}
        f(y)\geq f(x)+(y-x)^\top \nabla f(x).
    \end{align*}
    Choosing $y=x^*=\argmin_{x\in\calX}f(x)$ and using~\pref{prot:B-B}, we have
    \begin{align*}
        f(x^*)&\geq f(x)+(x^*-x)^\top \nabla f(x) \\
        &\geq f(x)+2(x^*-x)^\top \nabla f(x)+\frac{1}{B}(x-x^*)^\top \mathbb{E}_{f_t\sim \calF}\left[\nabla f_t(x)\nabla f_t(x)^\top\right](x-x^*).
    \end{align*}
    
    Therefore, we have
    \begin{align*}
        &\mathbb{E}\left[\sum_{t=1}^Tf_t(x_t)\right]-\mathbb{E}\left[\sum_{t=1}^Tf_t(x^*)\right] \\
        &\leq 2\mathbb{E}\left[\sum_{t=1}^T\inner{x_t-x^*, g_t}\right]-\mathbb{E}\left[\sum_{t=1}^T\frac{1}{B}(x_t-x^*)^\top g_tg_t^\top(x_t-x^*)\right] \\
        &\leq \order\left(\frac{J^2}{\eta T}+d\log (G T)\right)+\left(\frac{2}{\eta}-\frac{1}{B}\right)\mathbb{E}\left[\sum_{t=1}^T(x_t-x^*)^\top g_tg_t^\top(x_t-x^*)\right].
    \end{align*}
    Choosing $\eta\geq 2B$ leads to the bound.
\end{proof}

\subsection{Proof of~\pref{thm: low-switching-regret}}\label{app: lazy}
Finally, we prove that \pref{alg:ons-lazy}, a lazy version of~\pref{alg:ons} which only updates the decisions $\order(\log T)$ times over $T$ rounds, achieves $\order(\log^2 T)$ expected regret guarantee. This algorithm shares the same spirit of Algorithm 3 in~\citep{COLT21:lazyoco}. We highlight again that the low switching property is important to achieve $\order(\log^2 T)$ regret bound in our two-echelon inventory control problem.

\begin{theorem}\label{thm: low-switching-regret}
Let $\calX\subseteq\mathbb{R}^d$ be a convex set with bounded diameter $\max_{x,x'\in\calX}\|x-x'\|\leq J$. If $\{f_t\}_{t=1}^T$ satisfy~\pref{prot:B-B} for some $B>0$, $f_t:\calX\mapsto \mathbb{R}$ and $\max_{t\in[T], x\in\calX}\|\nabla f_t(x)\|\leq G$, then~\pref{alg:ons-lazy} with $\eta\geq 2B$, $\epsilon=\frac{1}{T}$ guarantees that
        $\mathbb{E}[\sum_{t=1}^Tf_t(w_t)]-\mathbb{E}[\sum_{t=1}^Tf_t(x^*)]\leq\order(\log^2 T+\log G+J^2/B)$, where $x^*=\argmin_{x\in\calX}\mathbb{E}_{f\sim \calF}[f(x)]$. Moreover, the decision sequence $\{w_t\}_{t=1}^T$ only switches $\order(\log T)$ times. 
\end{theorem}
\begin{proof}
    In~\pref{thm:ONS-B-B}, we know that for any $t\in[T]$, the decision sequence $\{x_s\}_{s=1}^t$ generated by ONS (\pref{alg:ons}) guarantees that,
    \begin{align*}
        \mathbb{E}\left[\sum_{s=1}^tf_s(x_s)\right]-\mathbb{E}\left[\sum_{s=1}^tf_s(x^*)\right]\leq \order(d\log GT+J^2/2BT),
    \end{align*}
    where $x^*=\argmin_{f\sim\calF}f(x)$. For any fixed $t$, using the convexity and stochasticity of $f_t$, we know that
    \begin{align*}
        \mathbb{E}\left[f_t(\wh{x}_k)-f_t(x^*)\right]\leq \frac{1}{2^k}\sum_{s=1}^{2^k}\mathbb{E}\left[f_t(x_s)-f_t(x^*)\right]=\frac{1}{2^k}\sum_{s=1}^{2^k}\mathbb{E}\left[f_s(x_s)-f_s(x^*)\right]\leq \order\left(\frac{1}{2^k}\cdot\left(d\log (G2^k)+\frac{J^2}{2^{k+1}B}\right)\right).
    \end{align*}
    
    Taking a summation over all $t\in [T]$, we have
    \begin{align*}
        \mathbb{E}\left[\sum_{t=1}^Tf_t(w_t)\right]-\mathbb{E}\left[\sum_{t=1}^Tf_t(x^*)\right]=\mathbb{E}\left[\sum_{k=0}^{\log_2 T}\sum_{t\in I_k}(f_t(\wh{w}_k)-f_t(x^*))\right]\leq \order\left(\log^2 T+\log G+\frac{J^2}{B}\right),
    \end{align*}
    where $I_k$ is the set of time index in the $k$-th epoch $[2^{k-1},2^k-1]$.
\end{proof}

\begin{lemma}\label{lem: high-prob-concentration-lemma}
Suppose that $\{f_t\}_{t=1}^T$ is a sequence of i.i.d convex functions drawn from a distribution $\calF$ and each $f_t:\calX\mapsto \mathbb{R}$ has the same bounded feasible domain $\max_{x,x'\in \calX}|x-x'|\leq J$. Let $x^*=\argmin_{x\in \calX}f(x)$ where $f(x)=\mathbb{E}_{f_t\sim \calF}[f_t(x)]$ and suppose that $f(x)$ is $\sigma$-strongly convex. Suppose that $\{x_t\}_{t=1}^T$ be the decision sequence~\pref{alg:ons} generates when the loss function sequence is $\{f_t\}_{t=1}^T$. Suppose that $\mathbb{E}\left[\sum_{t=1}^Tf(x_t)-f(x^*)\right]\leq R$. Let $\bar{x}_1=\frac{2}{T}\sum_{t=1}^{T/2}x_t$ and $\bar{x}_2=\frac{1}{T}\sum_{t=1}^{T}x_t$. Then, with probability at least $1-\delta$, we have
\begin{align*}
    |\bar{x}_1-x^*|\leq \order\left(\sqrt{\frac{R}{T\sigma}}+J\sqrt{\frac{\log\frac{1}{\delta}}{T}}\right),\\
    |\bar{x}_2-x^*|\leq \order\left(\sqrt{\frac{R}{T\sigma}}+J\sqrt{\frac{\log\frac{1}{\delta}}{T}}\right).
\end{align*}
\end{lemma}
\begin{proof}
    According to the strong convexity of $f(x)$, we know that
    \begin{align*}
        \frac{\sigma}{2}\mathbb{E}\left[\sum_{t=1}^{T/2}|x_t-x^*|^2\right]\leq \mathbb{E}\left[\sum_{t=1}^{T/2}(f(x_t)-f(x^*))\right]\leq \mathbb{E}\left[\sum_{t=1}^{T}(f(x_t)-f(x^*))\right]\leq R.
    \end{align*}
    By Cauchy-Schwarz inequality and the fact that $\mathbb{E}[x^2]\geq \mathbb{E}[x]^2$, we have
    \begin{align*}
        \mathbb{E}\left[\sum_{t=1}^{T/2}|x_t-x^*|\right]\leq \order\left(\sqrt{\frac{RT}{\sigma}}\right).
    \end{align*}
    According to Azuma's inequality and the boundedness of $x_t,x^*$, we have with probability at least $1-\frac{\delta}{2}$, 
    \begin{align}\label{eqn: azuma}
        \sum_{t=1}^{T/2}|x_{t}-x^*|-\mathbb{E}\left[\sum_{t=1}^{T/2}|x_{t}-x^*|\right]\leq \order\left(J\sqrt{T\log\frac{1}{\delta}}\right).
    \end{align}
    Note that $\bar{x}_1=\frac{2}{T}\sum_{t=1}^{T/2}x_{t}$. Therefore, with probability at least $1-\delta$,
    \begin{align*}
        \frac{T}{2}|\bar{x}_1-x^*|\leq \sum_{t=1}^{T/2}|x_{t}-x^*|\leq \order\left(J\sqrt{T\log\frac{1}{\delta}}+\sqrt{\frac{RT}{\sigma}}\right).
    \end{align*}
    Applying a similar analysis on $\bar{x}_2$ and a union bound finishes the proof.
\end{proof}

\subsection{Algorithm for the contract maker}\label{app: contract-maker}
We show the pseudo code of the algorithm for the contract maker in~\pref{alg:ogd-third-party-contract}.
\begin{algorithm}[t]
   \caption{Contract maker}
   \label{alg:ogd-third-party-contract}
    \textbf{Input:} A set of realized demand value $S=\{d_1,\ldots,d_L\}$, learning rate $\eta$.
    
    Construct empirical cumulative density function $\wh{\Phi}_L(\cdot)$ using $S$.
    
    Let $s_L$ be the output of~\pref{alg:ogd-third-party} with input $\calD$, $\wh{\Phi}_L(\cdot)$ and $\eta$.
   
   \Return $\omega_L=\frac{h_2\wh{\Phi}_L(s_L)}{1-\wh{\Phi}(s_L)}$.
\end{algorithm}

\subsection{Proof of~\pref{lem: beta-estimation}}\label{app: beta-estimation}
\begin{proof}
According to~\pref{lem: s-2-estimation}, we know that with probability at least $1-\delta$, $|\bar{s}_{L,2}-s_2^*|\leq \otil(L^{-\frac{1}{4}})$, where $\bar{s}_{L,2}$ is the output of~\pref{alg:ogd-third-party}.
To bound the difference between the contract coefficient $\omega$ returned by the third party and the optimal $\omega^*$, note that $\omega=\frac{h_2\wh{\Phi}_L(\bar{s}_{L,2})}{1-\wh{\Phi}_L(\bar{s}_{L,2})}$. According to~\pref{lem: DKW}, with probability at least $1-\delta$, it holds that
\begin{align}
    \left|\wh{\Phi}_L(\bar{s}_{L,2})-\Phi(s_2^*)\right| &\leq \left|\wh{\Phi}_L(\bar{s}_{L,2})-\Phi_L(\bar{s}_{L,2})\right|+ \left|\Phi(\bar{s}_{L,2})-\Phi(s_{2}^*)\right|\nonumber\\
    &\leq \sqrt{\frac{1}{2L}\ln\frac{2}{\delta}} + \phiup\left|\bar{s}_{L,2}-s_2^*\right| \nonumber\\
    &\leq C_3\log(T/\delta)L^{-\frac{1}{4}},\label{eqn:aux-1}
\end{align}
where the last inequality is due to~\pref{lem: s-2-estimation} and $C_3>0$ is a universal constant. 
Moreover, note that according to~\pref{eqn: s-2-star-bound}, we know that $\Phi(s_2^*)\leq \frac{p_1}{h_2+p_1}$, meaning that $\frac{1}{1-\Phi(s_2^*)}\leq \frac{h_2+p_1}{h_2}$. Combining with~\pref{eqn:aux-1}, we know that with probability $1-\delta$
\begin{align*}
    1-\wh{\Phi}_L(\bar{s}_{L,2})\geq 1 - \Phi(s_2^*) - C_3\log(T/\delta)L^{-\frac{1}{4}}\geq \frac{h_2}{h_2+p_1}- C_3\log(T/\delta)L^{-\frac{1}{4}}\geq \frac{h_2}{2(h_2+p_1)},
\end{align*}
where the last inequality holds when $L\geq C_4\triangleq\frac{16(h_2+p_1)^4C_3^4\log^4(T/\delta)}{h_2^4}$. Also it holds that $\omega^*=\frac{h_2\Phi(s_2^*)}{1-\Phi(s_2^*)}\leq \frac{h_2}{1-\frac{p_1}{h_2+p_1}}=h_2+p_1$.

Therefore, we obtain that
\begin{align*}
    |\omega-\omega^*|\leq\left|\frac{h_2\wh{\Phi}_L(\bar{s}_{L,2})}{1-\wh{\Phi}_L(\bar{s}_{L,2})}-\frac{h_2\Phi(s_2^*)}{1-\Phi(s_2^*)}\right| \leq h_2\left|\frac{\wh{\Phi}_L(\bar{s}_{L,2})-\Phi(s_2^*)}{(1-\wh{\Phi}_L(\bar{s}_{L,2}))(1-\Phi(s_2^*))}\right| \leq \order(L^{-\frac{1}{4}}\log(T/\delta)),
\end{align*}
which further shows that $\omega\in [0, \omega^*+ \order(L^{-\frac{1}{4}}\log(T/\delta))]\subseteq[0,h_2+p_1+\order(L^{-\frac{1}{4}}\log(T/\delta))]$.
\end{proof}

\subsection{Proof of~\pref{thm: decentralized-non-decoupling}}\label{app: decentralized-non-decoupling}
\begin{proof}
We first show the convergence on the inventory level decisions of Agent 2. We consider each epoch $I_m$ separately. As Agent 1 keeps his desired inventory level within each epoch, according to~\pref{lem: real-demand-coupling}, there are only constant number of rounds at the beginning of epoch $I_m$ such that $o_t\ne d_t$. 
With a slight abuse of notation, define
\begin{align*}
    \wh{H}_{t,2}^{\contract}(s_2)=h_2(s_2-d_t)^++\omega_m(s_2-d_t)^-,
\end{align*}
for $t\in I_m$. According to~\pref{lem: decouple-central-BB}, we know that $\wh{H}_{t,2}^{\contract}$ satisfies~\pref{prot:B-B} with a specific choice of $B>0$. Therefore, according to~\pref{alg:ons-lazy-three-non-decoupling} and~\pref{lem: decouple-central-BB}, Agent 2 is  using~\pref{alg:ons-lazy} within each epoch with respect to $\wh{H}_{t,2}^{\contract}$ except for constant number of rounds at the beginning of the epoch. According to~\pref{thm: low-switching-regret} and~\pref{lem: real-demand-coupling},
we know that the expected regret of Agent 2 within epoch $I_m$ is bounded as follows: picking $\delta=\frac{1}{T^2}$, for any $s_2\in [\dlow,\dup]$,
\begin{align}\label{eqn: interval-agent-2}
    \mathbb{E}\left[\sum_{t\in I_m}\left(\wt{H}_{t,2}^{\contract}-\bench_{t,2}^{\contract}(s_2)\right)\right] \leq \order\left(\log^2 T+\log T\right) + \order(1) = \order(\log^2 T).
\end{align}

As the total regret is upper bounded by the sum of the regrets in each epoch $m\in[M]$, $M=\order(\log T)$, we know that
\begin{align*}
    \mathbb{E}\left[\Reg_{T,2}\right]= \mathbb{E}\left[\sum_{m=1}^{M}\sum_{t\in I_m}\left(\wt{H}_{t,2}^{\contract}-\bench_{t,2}^{\contract}(\benchs_2^*)\right)\right]\leq \order(\log^3 T).
\end{align*}

Next, we consider the convergence of $s_{T,2}$ and bound the term $|s_{T,2}-s_2^*|$. More generally, let $e_m$ be the last round of epoch $m$ and we bound $|s_{e_m,2}-s_2^*|$. First, according to the analysis in~\pref{lem: beta-estimation} with a union bound, we know that if $\omega_m$ is generated by~\pref{alg:ogd-third-party}, with probability at least $1-\delta$, for any epoch index $m\in[M]$,
\begin{align}\label{eqn: beta-convergence}
    |\omega_m-\omega^*|\leq \order\left(L_m^{-\frac{1}{4}}\log \frac{T}{\delta}\right).
\end{align}
In addition, note that according to the dynamic of Agent 1 and Agent 2, there are $\Theta(2^m\cdot L_1)$ rounds in the epoch $I_m$ where Agent 1 keeps choosing her inventory level to be $s_{m,1}$, Agent 2 keeps choosing her inventory level to be $s_{e_m,2}$ and the contract coefficient is $\omega_m$. Define the set of these rounds to be $\calT_m$. In addition, define the expected loss function $H_{t,2}^{\contract}(s_2)=h_2\mathbb{E}_{x\sim\calD}\left[(s_2-x)^+\right]+\omega_m\mathbb{E}_{x\sim\calD}\left[(s_2-x)^-\right]$ for $t\in \calT_m$ and $\bar{s}_{m,2}^*=\argmin_{s_2}H_{t,2}^{\contract}(s_2)=\Phi^{-1}\left(\frac{\omega_m}{\omega_m+h_2}\right)$, where the second equality is by direct calculation. In addition, recall that $\wh{H}_{t,2}^{\contract}(s_2)=h_2(s_2-d_t)^++\omega_m(s_2-d_t)^-$. Then by choosing $\delta=\frac{1}{T^2}$, we know that for all $m\in[M]$:
\begin{align}
    & \mathbb{E}\left[\sum_{t\in \calT_m}H_{t,2}^{\contract}(s_{e_m,2})-\sum_{t\in \calT_m}H_{t,2}^{\contract}(\bar{s}_{m,2}^*)\right]\nonumber\\
    &\leq \mathbb{E}\left[\sum_{t\in I_m}H_{t,2}^{\contract}(s_{t,2})-\sum_{t\in I_m}H_{t,2}^{\contract}(\bar{s}_{m,2}^*)\right]\label{eqn: s_t_m_convergence}\\
    &= \mathbb{E}\left[\sum_{t\in I_m}\wh{H}_{t,2}^{\contract}(s_{t,2})-\sum_{t\in I_m}\wh{H}_{t,2}^{\contract}(\bar{s}_{m,2}^*)\right] \nonumber\\
    &\leq \order(\log^2 T). \tag{\pref{lem: real-demand-coupling} and~\pref{thm: low-switching-regret}}\\
     \nonumber
\end{align}
In addition, according to~\pref{lem: strong-convex-loss}, we know that $H_{t,2}^{\contract}(s_2)$ is strongly convex in $s_2$ with parameter $\sigma_m=\gamma(h_2+\omega_m)$. Therefore, according to~\pref{lem: high-prob-concentration-lemma}, we have with probability at least $1-\delta$, for all $m\in[M]$,
\begin{align}\label{eqn: s_T-s_2-bar}
    \left|s_{e_m,2}-\bar{s}_{m,2}^*\right|\leq \order\left(L_m^{-\frac{1}{2}}\log(T/\delta)\right)
\end{align}

Now we are ready to bound $|s_{e_m,2}-s_2^*|$. Recall that $s_2^*=\Phi^{-1}(\omega^*/(\omega^*+h_2))$. Therefore, with probability at least $1-2\delta$, for all $m\in[M]$,
\begin{align}
    |s_{e_m,2}-s_2^*|&\leq |s_{e_m,2}-\bar{s}_{m,2}^*|+|\bar{s}_{m,2}^*-s_2^*|\nonumber\\
    &\leq \order(L_m^{-\frac{1}{2}}\log(T/\delta)) + \left|{\Phi}^{-1}\left(\frac{\omega_m}{\omega_m+h_2}\right)-\Phi^{-1}\left(\frac{\omega^*}{\omega^*+h_2}\right)\right| \tag{\pref{eqn: s_T-s_2-bar}}\\
    &\leq \order(L_m^{-\frac{1}{2}}\log(T/\delta))+\frac{1}{\philow}\left|\frac{\omega_m}{\omega_m+h_2}-\frac{\omega^*}{\omega^*+h_2}\right| \tag{\pref{assum:bounded-density}}\\
    &\leq \order(L_m^{-\frac{1}{2}}\log(T/\delta))+\order(|\omega_m-\omega^*|)\nonumber\\
    &\leq \order(L_m^{-\frac{1}{4}}\log(T/\delta)), \label{eqn: s_T_2-converge}
\end{align}
where the last inequality is due to~\pref{eqn: beta-convergence}. Applying $m=M$ shows that $|s_{e_M,2}-s_2^*|=|s_{T,2}-s_2^*|\leq \order(T^{-\frac{1}{4}}\log(T/\delta))$, which finishes the proof for the convergence of Agent 2. 

In addition, according to~\pref{eqn: s_t_m_convergence} and Cauchy-Schwarz inequality, we know that within epoch $I_m$,
\begin{align}\label{eqn: sum_diff}
    \frac{\sigma_m}{2}\mathbb{E}\left[\sum_{t\in I_m}|s_{t,2}-\bar{s}_{m,2}^*|^2\right]\leq \order(\log^2 T) \Rightarrow \mathbb{E}\left[\sum_{t\in I_m}|s_{t,2}-\bar{s}_{m,2}^*|\right]\leq \order(\sqrt{L_m}\log T).
\end{align}
In addition, based on the boundedness of $s_{t,2}$ and $\bar{s}_{m,2}^*$, according to Hoeffding-Azuma's inequality, similar to~\pref{eqn: azuma}, with probability at least $1-\delta$, we know that for all $m\in[M]$,
\begin{align}\label{eqn:s-2-epoch-m-bar-star}
    \sum_{t\in I_m}|s_{t,2}-\bar{s}_{m,2}^*| - \mathbb{E}\left[\sum_{t\in I_m}|s_{t,2}-\bar{s}_{m,2}^*|\right] \leq \order\left(\sqrt{L_m\log\frac{M}{\delta}}\right).
\end{align}

Therefore, combining~\pref{eqn:s-2-epoch-m-bar-star} and~\pref{eqn: s_T_2-converge}, with probability at least $1-\delta$, for all $m\in[M]$,
\begin{align}
    \sum_{t\in I_m}|s_{t,2}-s_{2}^*|&\leq \sum_{t\in I_m}\left(|s_{t,2}-\bar{s}_{m,2}^*|+ |\bar{s}_{m,2}^*-s_2^*|\right)\nonumber\\
    &\leq \order\left(\sqrt{L_m\log\frac{M}{\delta}}\right)+\order\left(L_m^{\frac{3}{4}}\log(T/\delta)\right)=\order\left(L_m^{\frac{3}{4}}\log(T/\delta)\right).\label{eqn: sum_diff_epoch}
\end{align}

For Agent $1$, as $s_{m,1}=\wh{\Phi}_{m-1}^{-1}\left(\frac{h_2+p_1}{h_1+p_1}\right)$, using~\pref{lem: hoeffding}, we know that with probability at least $1-\delta$, for any $m\in [M]$,
\begin{align}\label{eqn: decentralized-s-1-convergence}
    |s_{m,1}-s_1^*| = \left|\wh{\Phi}_{m-1}^{-1}\left(\frac{h_2+p_1}{h_1+p_1}\right)-\Phi^{-1}\left(\frac{h_2+p_1}{h_1+p_1}\right)\right|\leq C_0\sqrt{\frac{\log(2TD/\delta)}{L_{m-1}}}=\order\left(\sqrt{\frac{\log(2T\dup/\delta)}{2^{m}}}\right) 
\end{align}
Again, setting $m=M$ proves the convergence of $s_{M,1}$.

Finally, we analyze the regret of Agent 1. For $t\in I_m$, define
\begin{align*}
    &\wh{H}_{t,1}^{\contract}(s_{1},s_{2})=h_1(\benchs_{t,1}-d_t)^++p_1(\benchs_{t,1}-d_t)^--\omega_m(s_{2}-d_t)^-,\\
    &H_{t,1}^{\contract}(s_{1},s_{2})=\mathbb{E}_{x\sim\calD}\left[h_1(\wh{s}_{t,1}-x)^++p_1(\wh{s}_{t,1}-x)^-\right]-\omega_m\mathbb{E}_{x\sim\calD}\left[(s_{2}-x)^-\right],
\end{align*}
where $\benchs_{t,1}=s_1$ if $s_2>d_{t-1}$ and $\benchs_{t,1}=s_1+s_2-d_{t-1}$ otherwise.
With the choice $\delta=\frac{1}{T^2}$, direct calculation shows that, for all $s_1$, $\wh{H}_{t,1}^{\contract}(s_1,\cdot)$ and $H_{t,1}^{\contract}(s_1,\cdot)$ are $\order(\log T)$-Lipschitz according to~\pref{lem: beta-estimation}. Based on~\pref{lem: real-demand-coupling}, we know that within each epoch, except for constant number of rounds, Agent 1 can achieve her intended inventory level and $o_t=d_t$. Therefore, by choosing $\delta=\frac{1}{T^2}$, we know that
\begin{align*}
    &\mathbb{E}\left[\Reg_{T,1}\right]\\
    &=\mathbb{E}\left[\sum_{t=1}^T\wt{H}_{t,1}^{\contract}\right]-\mathbb{E}\left[\sum_{t=1}^T\bench_{t,1}^{\contract}(\benchs_1^*)\right]\\
    &\leq \mathbb{E}\left[\sum_{m=1}^M\sum_{t\in I_m}\wh{H}_{t,1}^{\contract}(s_{m,1},s_{t,2})\right]-\min_{s_1}\mathbb{E}\left[\sum_{t=1}^T\wh{H}_{t,1}^{\contract}(s_{1}, s_{t,2})\right] +\otil(1) \tag{\pref{lem: real-demand-coupling}}\\
    &\leq \mathbb{E}\left[\sum_{m=1}^M\sum_{t\in I_m}\wh{H}_{t,1}^{\contract}(s_{m,1},s_{2}^*)\right]-\min_{s_1}\mathbb{E}\left[\sum_{t=1}^T\wh{H}_{t,1}^{\contract}(s_{1}, s_{2}^*)\right]+\mathbb{E}\left[\sum_{m}\sum_{t\in I_m}\otil(|s_{t,2}-s_2^*|)\right]+\otil(1) \tag{Lipschitzness of $\wh{H}_{t,1}^{\contract}(s_1,\cdot)$}\\
    &\leq \mathbb{E}\left[\sum_{m=1}^M\sum_{t\in I_m}H_{t,1}^{\contract}(s_{m,1},s_2^*)\right]-\min_{s_1}\mathbb{E}\left[\sum_{m=1}^M\sum_{t\in I_m}H_{t,1}^{\contract}(s_1,s_2^*)\right] + \otil(T^{\frac{3}{4}}) \tag{\pref{eqn: sum_diff_epoch}}\\
    &\leq \mathbb{E}\left[\sum_{m=1}^M\sum_{t\in I_m}H_{t,1}^{\contract}(s_{m,1},s_2^*)\right]-\mathbb{E}\left[\sum_{m=1}^M\sum_{t\in I_m}H_{t,1}^{\contract}(s_1^*,s_2^*)\right] + \otil(T^{\frac{3}{4}})
     \tag{$s_1^*$ is the minimizer of ${H}_{t,1}^{\contract}(s_1,s_{2}^*)$}\\
    &\leq \mathbb{E}\left[\sum_{m=1}^M\otil\left(2^m\cdot |s_{m,1}-s_1^*|\right)\right] +\otil(T^{\frac{3}{4}}) \tag{$\max\{h_1,p_1\}$-Lipschitzness of $H_{t,1}^{\contract}(\cdot,s_2^*)$}\\
    &\leq \mathbb{E}\left[\sum_{m=1}^M\otil\left(\sqrt{2^m}\right)\right] +\otil(T^{\frac{3}{4}})\tag{\pref{eqn: decentralized-s-1-convergence}}\\
    &\leq \otil(T^{\frac{3}{4}}),
\end{align*}
which finishes the proof.
\end{proof}

\subsection{Proof of~\pref{thm: coupling-overall-regret}}\label{app: coupling-overall-regret}
\begin{proof}
Note that from~\pref{eqn: reg_central} and the convexity of $H(s_1^*, s_2)$ in $s_2$, let $\bar{s}_{m,2}$ be the output of the contract maker~\pref{alg:ogd-third-party}, we know that
\begin{align}\label{eqn:reg-s-bar}
    \mathbb{E}\left[\sum_{t\in I_m}H(s_1^*, \bar{s}_{m,2})-\sum_{t\in I_m}H(s_1^*,s_2^*)\right] \leq \order\left(\sqrt{L_m\log T}\right).
\end{align}

Let $\bar{s}_{m,2}^*=\Phi^{-1}(\frac{\omega_m}{\omega_m+h_2})$. First, we bound $\sum_{t\in I_m}|s_{t,2}-\bar{s}_{m,2}^*|$ for each epoch $m$. Note that according to the analysis in~\pref{eqn: s_t_m_convergence}, we know that
\begin{align*}
    \mathbb{E}\left[\sum_{t\in I_m}H_{t,2}^{\contract}(s_{t,2})-\sum_{t\in I_m}H_{t,2}^{\contract}(\bar{s}_{m,2}^*)\right]\leq \order(\log^2 T).
\end{align*}
According to~\pref{eqn: sum_diff} and~\pref{eqn:s-2-epoch-m-bar-star}, we know that with probability $1-\delta$, for all $m\in[M]$,
\begin{align}\label{eqn:s-t-2-s-bar-star}
    \sum_{t\in I_m}\left|s_{t,2}-\bar{s}_{m,2}^*\right|\leq \order\left(\sqrt{L_m}\log\frac{T}{\delta}\right).
\end{align}

Next, we bound $\sum_{t\in I_m}|\bar{s}_{m,2}^*-\bar{s}_{m,2}|$ for each epoch $m$. Define $\wt{s}_{m,2} =\wh{\Phi}_{m-1}^{-1}(\frac{\omega_m}{\omega_m+h_2})$. Note that $\omega_m=\frac{h_2\wh{\Phi}_{m-1}(\bar{s}_{m,2})}{1-\wh{\Phi}_{m-1}(\bar{s}_{m,2})}$. Let $\{d_k\}_{k=1}^{L_{m-1}}$ be the demand samples realized in epoch $I_{m-1}$ and let $\{d_k'\}_{k=1}^{L_{m-1}}$ be the sorted sequence in non-decreasing order. Then, with probability at least $1-\delta$, for each $m\in[M]$,
\begin{align}
    &\sum_{t\in I_m}|\wt{s}_{m,2}-\bar{s}_{m,2}| = L_m\cdot \left|\wh{\Phi}_{m-1}^{-1}\left(\wh{\Phi}_{m-1}(\bar{s}_{m,2})\right)-\bar{s}_{m,2}\right| \nonumber\\
    &\leq L_m\cdot\max_{k\in [L_{m-1}]}|d_k'-d_{k-1}'| \nonumber\\
    &\leq L_m\cdot\frac{1}{\philow} \max_{k\in [L_{m-1}]}\left|\Phi(d_k')-\Phi(d_{k-1}')\right| \nonumber\\
    &\leq \frac{2}{\philow}\log \frac{L_{m-1} MT}{\delta},\label{eqn:s-tilde-s-bar}
\end{align}
where the last inequality is due to~\pref{eqn: order-stats-gap}.

Then, we bound the term $\sum_{t\in I_m}|\bar{s}_{m,2}^*-\wt{s}_{m,2}|$. According to~\pref{lem: hoeffding}, with probability at least $1-\delta$, for each $m\in[M]$,
\begin{align}\label{eqn:s-bar-star-s-tilde}
    \sum_{t\in I_m}|\bar{s}_{m,2}^*-\wt{s}_{m,2}| = L_m\left|\Phi^{-1}\left(\frac{\omega_m}{\omega_m+h_2}\right)-\wh{\Phi}_{m-1}^{-1}\left(\frac{\omega_m}{\omega_m+h_2}\right)\right|\leq \otil\left(\frac{L_m}{\gamma}\sqrt{\frac{\log(1/\delta)}{L_{m-1}}}\right).
\end{align}

Therefore, according to the Lipschitzness of $H(s_1,s_2)$ in both parameters, picking $\delta=\frac{1}{T^3}$, we can obtain that
\begin{align*}
    &\mathbb{E}\left[\sum_{t\in I_m}H(s_1^*, s_{t,2})-\sum_{t\in I_m}H(s_1^*,s_2^*)\right]\\
    &= \mathbb{E}\left[\sum_{t\in I_m}H(s_1^*, s_{t,2})-\sum_{t\in I_m}H(s_1^*,\bar{s}_{m,2}^*)\right] + \mathbb{E}\left[\sum_{t\in I_m}H(s_1^*, \bar{s}_{m,2}^*)-\sum_{t\in I_m}H(s_1^*,\wt{s}_{m,2})\right] \\
    &\qquad + \mathbb{E}\left[\sum_{t\in I_m}H(s_1^*, \wt{s}_{m,2})-\sum_{t\in I_m}H(s_1^*,\bar{s}_{m,2})\right] + \mathbb{E}\left[\sum_{t\in I_m}H(s_1^*, \bar{s}_{m,2})-\sum_{t\in I_m}H(s_1^*,s_2^*)\right]\\
    & \leq \sum_{t\in I_m}\left(\otil\left(|s_{t,2}-\bar{s}_{m,2}^*|\right) + \otil\left(|\bar{s}_{m,2}^*-\wt{s}_{m,2}|\right) + \otil\left(|\wt{s}_{m,2}-\bar{s}_{m,2}|\right)\right) + \order(1) \\
    &\qquad +
    \mathbb{E}\left[\sum_{t\in I_m}H(s_1^*,\bar{s}_{m,2})-\sum_{t\in I_m}H(s_1^*,s_2^*)\right] \\
    &\leq \otil(\sqrt{L_m}),
\end{align*}
where the last inequality is by combining~\pref{eqn:reg-s-bar},~\pref{eqn:s-t-2-s-bar-star},~\pref{eqn:s-tilde-s-bar} and~\pref{eqn:s-bar-star-s-tilde}. Finally, note that from~\pref{eqn: decentralized-s-1-convergence}, we know that for all $m\in[M]$, $\mathbb{E}[|s_{m,1}-s_1^*|]\leq \otil(1/\sqrt{L_m})$. Again using the Lipschitzness of $H(s_1,s_2)$, we can obtain that for all $m\in[M]$,
\begin{align*}
    \mathbb{E}\left[\sum_{t\in I_m}H(s_{m,1}, s_{t,2})-\sum_{t\in I_m}H(s_1^*,s_2^*)\right] \leq \otil(\sqrt{L_m}).
\end{align*}
Taking summation over all $m\in [M]$, we know that
\begin{align*}
    \mathbb{E}\left[\sum_{m=1}^M\sum_{t\in I_m}\left(H(s_{m,1},s_{t,2})-H(s_1^*,s_2^*)\right)\right]\leq \otil(\sqrt{T}).
\end{align*}
Finally, according to~\pref{lem: real-demand-coupling}, as both agents only changes their decision $\otil(1)$ number of rounds and within each epoch, we know that only constant number of round such that the desired inventory level can not be realized and $d_t\ne o_t$. Therefore, we can obtain that
\begin{align*}
    \mathbb{E}\left[\Reg_{T}\right]&= \mathbb{E}\left[\sum_{m=1}^M\sum_{t\in I_m}\left(\wt{H}_t-{H}(s_1^*,s_2^*)\right)\right]\\
    &\leq \mathbb{E}\left[\sum_{m=1}^M\sum_{t\in I_m}\left({H}(s_{m,1},s_{t,2})-{H}(s_1^*,s_2^*)\right)\right] + \otil(1) \leq \otil(\sqrt{T}),
\end{align*}
which finishes the proof.
\end{proof}
\section{Auxiliary lemmas}\label{app: auxiliary lemma}
In this section, we introduce several lemmas that are useful in the analysis. The first three lemmas show the properties of the empirical density function and the true density function. Suppose in epoch $I$ with $|I|=L$, we receive the demand $d_1,d_2,\dots,d_L$. 
Define the empirical cumulative density function $\hat{\Phi}_L(\cdot)$ constructed by $\{d_i\}_{i=1}^L$ as
\begin{align}\label{eqn: empirical density function L}
    \wh{\Phi}_L(x)=\frac{1}{L}\sum_{i=1}^L\mathbb{I}\{d_i \leq x\},
\end{align}
and the corresponding inverse cumulative density function $\hat{\Phi}_L^{-1}(\cdot)$:
\begin{align}\label{eqn: inverse-empirical-density}
    \hat{\Phi}_L^{-1}(\kappa) = \min\left\{z: \frac{1}{L}\sum_{i=1}^L\mathbb{I}\{d_i\leq z\}\geq \kappa\right\},
\end{align}
where $\kappa\in[0,1]$. The following Dvoretzky–Kiefer–Wolfowitz lemma shows the concentration between $\hat{\Phi}_L(a)$ and $\Phi(a)$ for any $a\in \mathbb{R}$.

\begin{lemma}{(Dvoretzky–Kiefer–Wolfowitz lemma)}\label{lem: DKW}
Let $\{d_i\}_{i=1}^T$ be $T$ i.i.d. samples drawn from distribution $\calD$ with cumulative density function $\Phi$. Define the empirical cumulative density function $\wh{\Phi}_L(\cdot)$ as shown in~\pref{eqn: empirical density function L}, $L\in[T]$. Then with probability at least $1-\delta$, for any $x\in \mathbb{R}$, 
\begin{align*}
    \left|\wh{\Phi}_L(x)-\Phi(x)\right|\leq \sqrt{\frac{1}{2L}\ln\frac{2}{\delta}}.
\end{align*}
Moreover, by applying a union bounded over all $L\in[T]$, with probability at least $1-\delta$, for any $x\in\mathbb{R}$ and $L'\in[T]$, it holds that
\begin{align*}
    \left|\wh{\Phi}_{L'}(x)-\Phi(x)\right|\leq \sqrt{\frac{1}{2L'}\ln\frac{2T}{\delta}}.
\end{align*}
\end{lemma}


The next lemma shows the stability of $\wh{\Phi}_L^{-1}(\cdot)$ on consecutive grids of length $\frac{1}{T}$ over $[0,1]$, which turns out to be important to prove our main lemma~\pref{lem: hoeffding}.

\begin{lemma}\label{lem: inverse-close}
Let $\{d_i\}_{i=1}^T$ be $T$ i.i.d. samples from distribution $\calD$ satisfying~\pref{assum:bounded-density}. Let $\hat{\Phi}_L(\cdot)$ be the empirical cumulative density function constructed by $\{d_i\}_{i=1}^L$ as shown in~\pref{eqn: empirical density function L}. The inverse of the empirical density function $\wh{\Phi}^{-1}(\cdot)$ is defined in~\pref{eqn: inverse-empirical-density}.
Then with probability at least $1-\delta$, for any $\kappa\in \left\{\frac{i}{T}\right\}_{i=0}^{T-1}$ and any $L\in [T]$,
\begin{align*}
    \wh{\Phi}_L^{-1}\left(\kappa+\frac{1}{T}\right)-\wh{\Phi}_{L}^{-1}\left(\kappa\right)\leq \frac{2}{\philow L}\ln\frac{LT}{\delta}.
\end{align*}
\end{lemma}
\begin{proof}
    Fix any $L\in[T]$. Without loss of generality, we assume that $d_1\leq d_2\leq \ldots\leq d_L$ be the ordered realized demand and let $d_0=0$. According to the definition of $\wh{\Phi}_L^{-1}$, we know that for $a\in (\frac{i-1}{L},\frac{i}{L}]$, $i\in [L]$
    \begin{align*}
        \wh{\Phi}_L^{-1}(a)=d_i.
    \end{align*}
    Moreover, as $L\leq T$, meaning that $[\kappa,\kappa+\frac{1}{T}]\subseteq (\frac{i-1}{L}, \frac{i+1}{L}]$ for some $i\in [L]$, we know that for any $\kappa\in \left\{\frac{i}{T}\right\}_{i=0}^{T-1}$,
    \begin{align*}
        \wh{\Phi}_L^{-1}\left(\kappa+\frac{1}{T}\right)-\wh{\Phi}_L^{-1}(\kappa) \leq \max_{i\in [L]}\left|d_i-d_{i-1}\right|.
    \end{align*}
    Note that according to the property of cumulative density function, $\Phi(x)$ with $x\sim\calD$ follows the uniform distribution $\mathcal{U}[0,1]$. According to the property of the ordered statistics of $\mathcal{U}[0,1]$, let $\Delta_k=\Phi(d_{k})-\Phi(d_{k-1})$ be the gap between the $k-1$-th and the $k$-th ordered statistics, $k\in [L]$ and we have $\Delta_k$ follows the Beta distribution $\Delta_k\sim \text{Beta}(1,L)$. Therefore, for any $k\in [L]$,
    \begin{align*}
        \P\left[\Delta_k\geq r\right] = \int_{r}^1L(1-u)^{L-1}du = (1-r)^L.
    \end{align*}
    Let $r=\frac{1}{L}\ln \frac{L}{\delta}$. Then we have
    \begin{align*}
        \P\left[\exists k\in [L], \Delta_k\geq r\right] &\leq\sum_{k=1}^L\P\left[\Delta_k\geq r\right] \\
        &= L(1-r)^L\\
        &\leq L\left(1-\frac{1}{L}\ln\frac{L}{\delta}\right)^L \\
        &\leq L\left(\left(1-\frac{1}{L}\ln\frac{L}{\delta}\right)^{\frac{L}{\ln \frac{L}{\delta}}}\right)^{\ln \frac{L}{\delta}} \\
        &\leq \delta.
    \end{align*}
    Therefore, with probability at least $1-\delta$, we have $\Delta_k\leq \frac{1}{L}\ln \frac{L}{\delta}$, for all $k\in [L]$. According to the assumption that $\phi(d)\geq \philow$, we have with probability $1-\delta$,
    \begin{align}\label{eqn: order-stats-gap}
        \max_{i\in [L]}|d_i-d_{i-1}|\leq \max_{i\in [L]}\frac{1}{\philow}\cdot\left|\Phi(d_i)-\Phi(d_{i-1})\right|\leq \frac{1}{\philow L}\ln\frac{L}{\delta}.
    \end{align}
    Taking a union bound over all possible choices of $\kappa$ and $L\in[T]$ gives the conclusion.
\end{proof}

Now we are ready to prove~\pref{lem: hoeffding}. Note that different from the concentration result which holds for a specific known $\kappa$ (e.g. Proposition $3$ in~\citep{MOR21:nonparametric}), with the help of~\pref{lem: inverse-close},~\pref{lem: hoeffding} proves that with high probability, for \emph{all} $\kappa\in [0,1]$, we have the difference between $\hat{\Phi}_L^{-1}(\kappa)$ and $\Phi^{-1}(\kappa)$ bounded by $\otil(1/\sqrt{L})$, which is what we require in the decentralized setting with contract in the coupling model introduced in~\pref{sec: coupling} as $\frac{\beta_m}{\beta_m+h_2}$ can take arbitrary values between $[0,1]$. 

\begin{lemma}\label{lem: hoeffding}
    Let $\{d_i\}_{i=1}^T$ be $T$ i.i.d. samples from distribution~$\calD$ which satisfies~\pref{assum:bounded-density}. Let $\wh{\Phi}_L(\cdot)$ be the empirical cumulative density function constructed by $\{d_i\}_{i=1}^L$ as shown in~\pref{eqn: empirical density function L}. The inverse of the empirical density function $\wh{\Phi}^{-1}(\cdot)$ is defined in~\pref{eqn: inverse-empirical-density}. Then with probability at least $1-\delta$, for any $\kappa\in [0,1]$ and any $L\in[T]$, it holds that
    \begin{align*}
        &\left|\wh{\Phi}_L^{-1}(\kappa)-\Phi^{-1}(\kappa)\right|\leq  C_0\sqrt{\frac{\ln\frac{TD}{\delta}}{L}},\\
        &\left|\Phi\left(\wh{\Phi}_L^{-1}(\kappa)\right)-\kappa\right|\leq C_0\sqrt{\frac{\ln \frac{TD}{\delta}}{L}},
    \end{align*}
    where $C_0>0$ are some universal constants.
\end{lemma}
\begin{proof} For any fixed $\kappa\in \{\frac{i}{T}\}_{i=0}^{T}$ and $L\in[T]$, we know that
\begin{align}
    &\P\left[\Phi\left(\hat{\Phi}_L^{-1}\left(\kappa\right)\right)-\kappa\leq -\xi\right] \nonumber \\
    &\leq \P\left[\hat{\Phi}_L^{-1}\left(\kappa\right)\leq \Phi^{-1}\left(\kappa-\xi\right)\right] \nonumber \\
    &\leq \P\left[\frac{1}{L}\sum_{i=1}^L\left\{d_i\leq \Phi^{-1}\left(\kappa-\xi\right)\right\}\geq \kappa\right] \tag{according to the definition of $\wh{\Phi}_L^{-1}(\cdot)$}\\
    &\leq \P\left[\frac{1}{L}\sum_{i=1}^L\left\{d_i\leq \Phi^{-1}\left(\kappa-\xi\right)\right\} - \left(\kappa-\xi\right)\geq \xi \right] \leq \exp(-2L\xi^2) \label{eqn: Phi diff},
\end{align}
where the last inequality is by Hoeffding's inequality. On the other hand,
\begin{align*}
    &\P\left[\Phi\left(\hat{\Phi}_L^{-1}\left(\kappa\right)\right)-\kappa\geq \xi\right] \\
    &\leq \P\left[\hat{\Phi}_L^{-1}\left(\kappa\right)\geq \Phi^{-1}\left(\kappa+\xi\right)\right] \\
    &\leq \P\left[\frac{1}{L}\sum_{i=1}^L\left\{d_i\leq \Phi^{-1}\left(\kappa+\xi\right)\right\}< \kappa\right] \tag{according to the definition of $\wh{\Phi}_L^{-1}(\cdot)$}\\
    &\leq \P\left[\frac{1}{L}\sum_{i=1}^L\left\{d_i\leq \Phi^{-1}\left(\kappa+\xi\right)\right\} - \left(\kappa+\xi\right)< -\xi \right] \leq \exp(-2L\xi^2),
\end{align*}
Therefore, we conclude that
\begin{align*}
    \P\left[\left|\Phi\left(\wh{\Phi}_L^{-1}(\kappa)\right)-\kappa\right|\geq \xi\right]\leq 2\exp(-2L\xi^2).
\end{align*}
Therefore, with probability at least $1-\delta$,
\begin{align*}
    \left|\Phi\left(\wh{\Phi}_L^{-1}(\kappa)\right)-\kappa\right|\leq \sqrt{\frac{\ln \frac{2}{\delta}}{2L}}.
\end{align*}
Taking a union bound over all $\kappa\in \{\frac{i}{T}\}_{i=0}^{T}$, with probability at least $1-\delta$, we can obtain that for all $\kappa\in \{\frac{i}{T}\}_{i=0}^{T}$,
\begin{align}\label{eqn: phi-bound}
    \left|\Phi\left(\wh{\Phi}_L^{-1}(\kappa)\right)-\kappa\right|\leq \sqrt{\frac{\ln \frac{2TD}{\delta}}{2L}}.
\end{align}
Next, for any $\kappa\in[0,1]$, let $\kappa_0\geq \kappa, \kappa_1\leq \kappa$ be the real number such that $\kappa_0-\kappa$ and $\kappa-\kappa_1$ is minimized and $\kappa_0,\kappa_1\in \{\frac{i}{T}\}_{i=0}^{T}$, $\kappa_0-\kappa_1=\frac{1}{T}$. Then, according to~\pref{lem: inverse-close}, with probability at least $1-\frac{\delta}{2}$, we have
\begin{align*}
    &\left|\wh{\Phi}_L^{-1}(\kappa)-\Phi^{-1}(\kappa)\right| \\
    &= \left|\wh{\Phi}_L^{-1}(\kappa)-\wh{\Phi}_L^{-1}(\kappa_1)+\wh{\Phi}_L^{-1}(\kappa_1)-\Phi^{-1}(\kappa_1)+\Phi^{-1}(\kappa_1)-\Phi^{-1}(\kappa)\right| \\
    &\leq \left|\wh{\Phi}_L^{-1}(\kappa)-\wh{\Phi}_L^{-1}(\kappa_1)\right|+\left|\wh{\Phi}_L^{-1}(\kappa_1)-\Phi^{-1}(\kappa_1)\right|+\left|\Phi^{-1}(\kappa_1)-\Phi^{-1}(\kappa_0)\right| \tag{$\Phi^{-1}(\kappa_1)\leq \Phi^{-1}(\kappa)\leq \Phi^{-1}(\kappa_0)$} \\
    &\leq \left|\wh{\Phi}_L^{-1}(\kappa_0)-\wh{\Phi}_L^{-1}(\kappa_1)\right|+\left|\wh{\Phi}_L^{-1}(\kappa_1)-\Phi^{-1}(\kappa_1)\right|+\left|\Phi^{-1}(\kappa_1)-\Phi^{-1}(\kappa_0)\right| \\
    &\leq \frac{2}{\philow L}\ln\frac{4LT}{\delta}+\frac{1}{\philow}\left|\Phi\left(\wh{\Phi}_L^{-1}(\kappa_1)\right)-\kappa_1\right|+\left|\Phi^{-1}(\kappa_1)-\Phi^{-1}(\kappa_0)\right| \tag{\pref{lem: inverse-close}}\\
    &\leq \frac{1}{\philow L}\ln\frac{LT}{\delta}+\frac{1}{\philow}\sqrt{\frac{\ln \frac{8TD}{\delta}}{2L}}+\frac{1}{\philow T} \tag{\pref{eqn: phi-bound}}\\
    &\leq C'\sqrt{\frac{\ln \frac{TD}{\delta}}{L}},
\end{align*}
where $C'>0$ is some universal constant.

For $\left|\Phi(\wh{\Phi}_L^{-1}(\kappa))-\kappa\right|$, define $\kappa_0$ and $\kappa_1$ the same as before and we know that with probability at least $1-\frac{\delta}{2}$,
\begin{align*}
    &\left|\Phi\left(\wh{\Phi}_L^{-1}(\kappa)\right)-\kappa\right| \\
    &= \left|\Phi\left(\wh{\Phi}_L^{-1}(\kappa)\right)-\Phi\left(\wh{\Phi}_L^{-1}(\kappa_1)\right)+\Phi\left(\wh{\Phi}_L^{-1}(\kappa_1)\right)-\kappa_1+\kappa_1-\kappa\right|\\
    &\leq \left|\Phi\left(\wh{\Phi}_L^{-1}(\kappa_0)\right)-\Phi\left(\wh{\Phi}_L^{-1}(\kappa_1)\right)\right|+\left|\Phi\left(\wh{\Phi}_L^{-1}(\kappa_1)\right)-\kappa_1\right|+\left|\kappa_1-\kappa\right|\\
    &\leq \phiup\left|\wh{\Phi}_L^{-1}(\kappa_0)-\wh{\Phi}_L^{-1}(\kappa_1)\right|+\sqrt{\frac{\ln \frac{2TD}{\delta}}{2L}}+\frac{1}{T} \tag{\pref{eqn: phi-bound} and Lipschitzness of $\Phi$}\\
    &\leq \frac{2\phiup}{\philow L}\ln\frac{4LT}{\delta}+\sqrt{\frac{\ln\frac{8TD}{\delta}}{2L}}+\frac{1}{T} \tag{\pref{lem: inverse-close}}\\
    &\leq C''\sqrt{\frac{\ln \frac{TD}{\delta}}{L}},
\end{align*}
where $C''>0$ is some universal constant. Taking a union bound over all $\kappa$ and $L\in[T]$ and choosing $C_0=\max\{4C',4C''\}$ finish the proof.
\end{proof}

The last lemma shows the strong convexity of the expectation of the loss function introduced in~\pref{eqn: loss-decouple-agent-2}.
\begin{lemma}\label{lem: strong-convex-loss}
For any $h>0, p>0$, let $f(s)=h\mathbb{E}_{x\sim \calD}\left[(s-x)^+\right]+p\mathbb{E}_{x\sim \calD}\left[(s-x)^-\right]$ where $\calD$ satisfies~\pref{assum:bounded} and~\pref{assum:bounded-density}. Then $f(s)$ is strongly convex in $s$ with strongly convex parameter $\sigma=(h+p)\gamma$ where $\gamma$ is defined in~\pref{assum:bounded-density}.
\end{lemma}
\begin{proof}
    Taking the second order gradient of $f(s)$, we know that
    \begin{align*}
        &\nabla^2f(s)=(h+p)\phi(x)\geq (h+p)\philow,
    \end{align*}
    where the last inequality is due to~\pref{assum:bounded-density}. This finishes the proof.
\end{proof}

\section{Experiments}\label{app: coupling-exp}
In this section, we show our empirical results for our designed algorithms. Specifically, we verify the empirical performance of~\pref{alg:central-lazy-three} and~\pref{alg:ons-lazy-three-non-decoupling} in our model. We construct different bounded demand distributions listed as follows: 1) Gaussian distribution $\calN(3,1)$ clipped on support $[1,4]$; 2) uniform distribution over $[1,4]$; 3) exponential distribution with mean $3$ clipped on support $[1,4]$. We also set the number of round to be $T=800000$ and choose the cost configuration to be $(h_1,h_2,p_1)=\{(0.3,0.1,0.5),(0.4,0.25,0.6),(0.5,0.35,0.75),(0.6,0.4,0.85)\}$. For each demand distribution, $128$ trials are processed and we calculate the mean and the standard deviation of the regret over the $128$ trials. The results are shown in~\pref{fig:coupling}. The results show the effectiveness of our proposed algorithm in both the centralized and decentralized setting.

\begin{figure}[!t]
\centering
\includegraphics[width=0.327\textwidth]{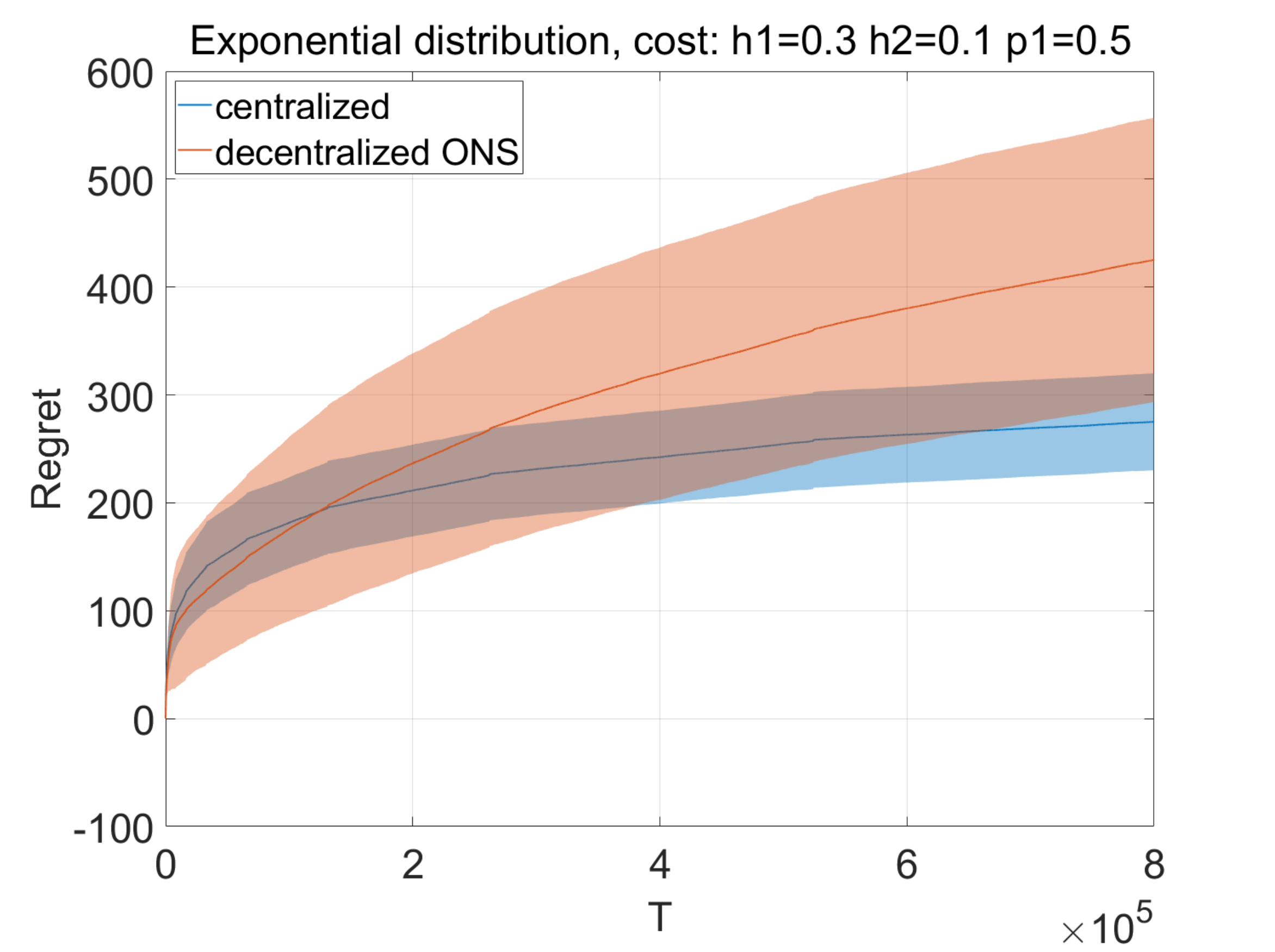}
\includegraphics[width=0.327\textwidth]{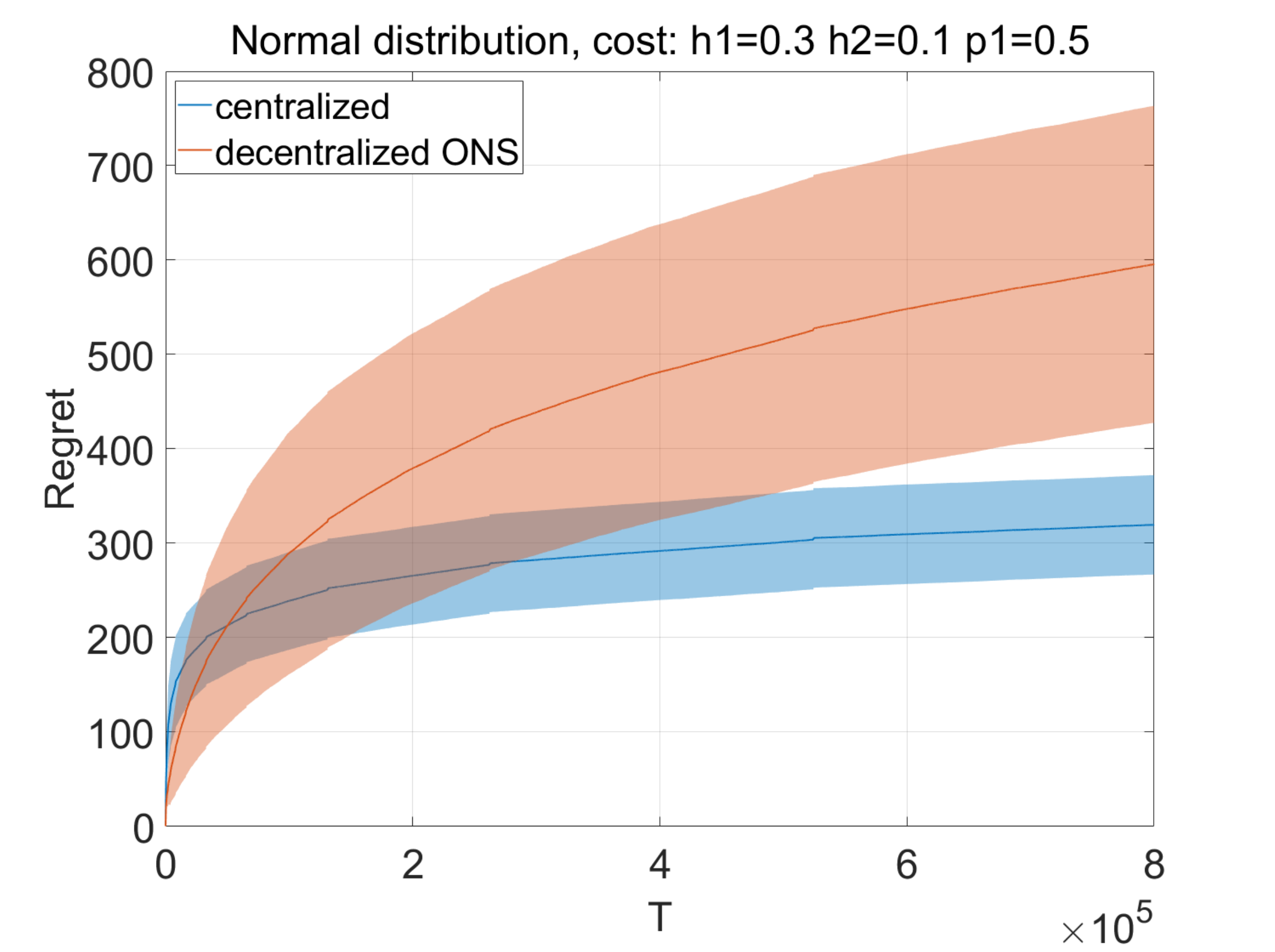}
\includegraphics[width=0.327\textwidth]{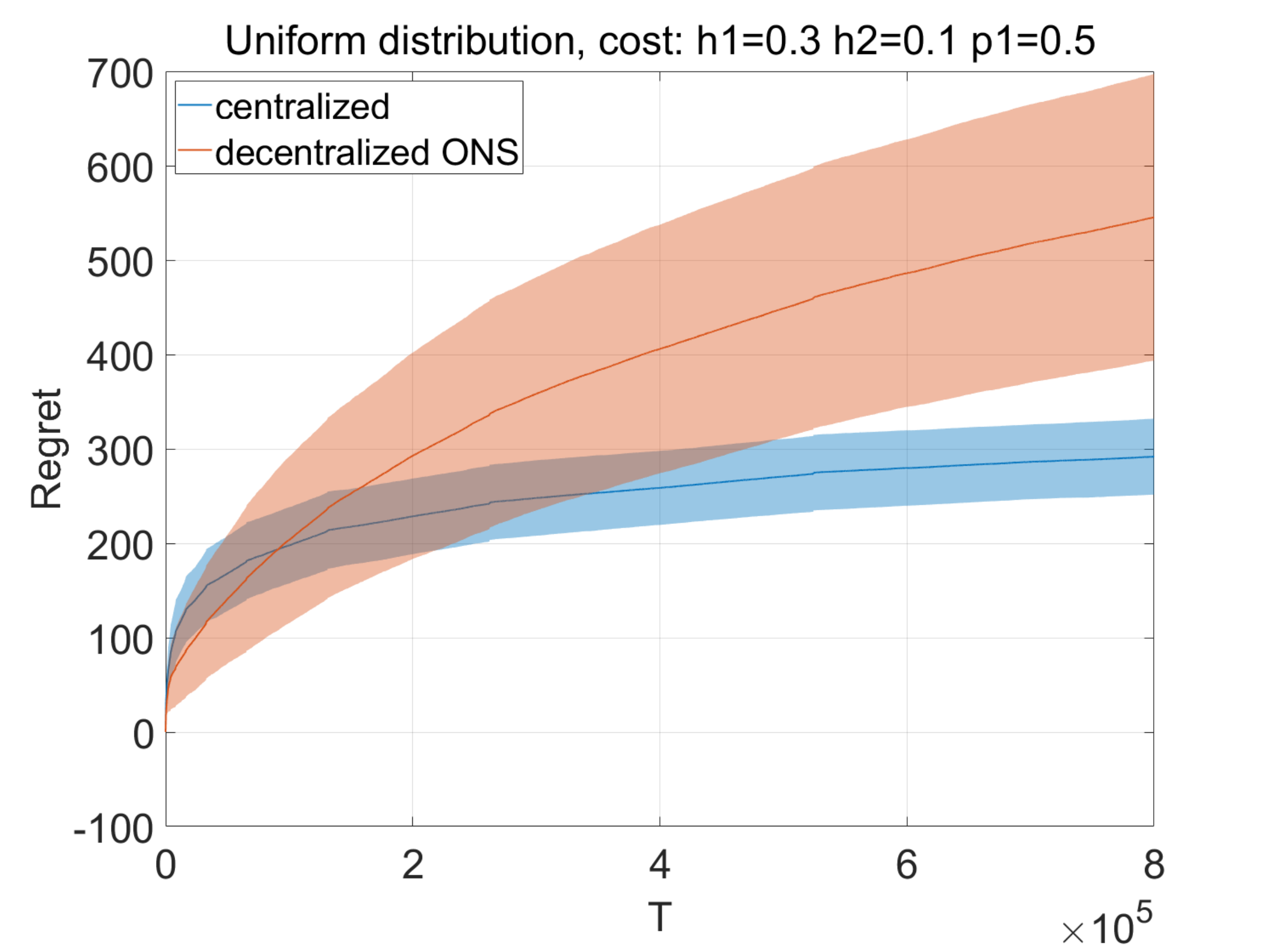}
\includegraphics[width=0.327\textwidth]{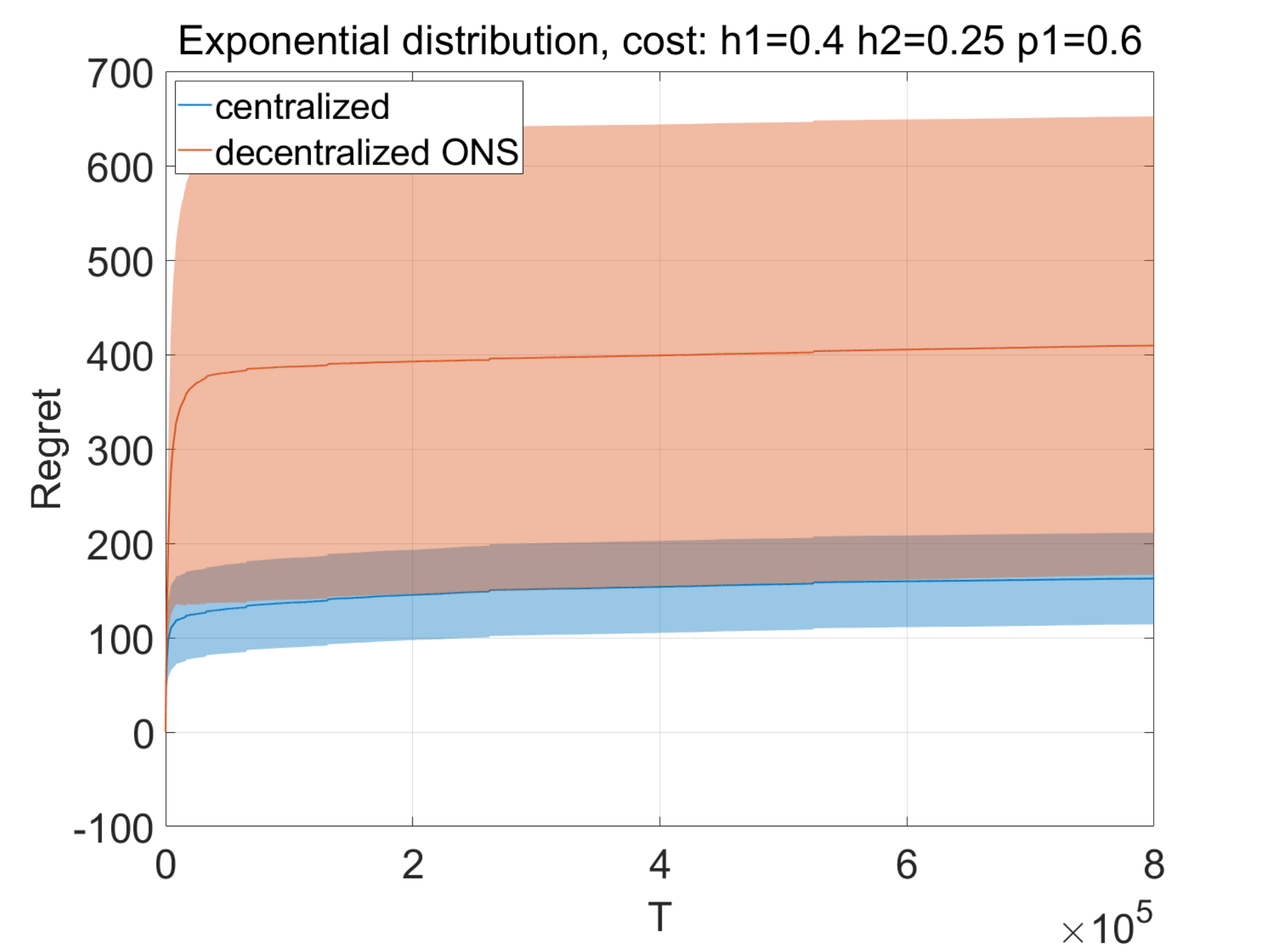}
\includegraphics[width=0.327\textwidth]{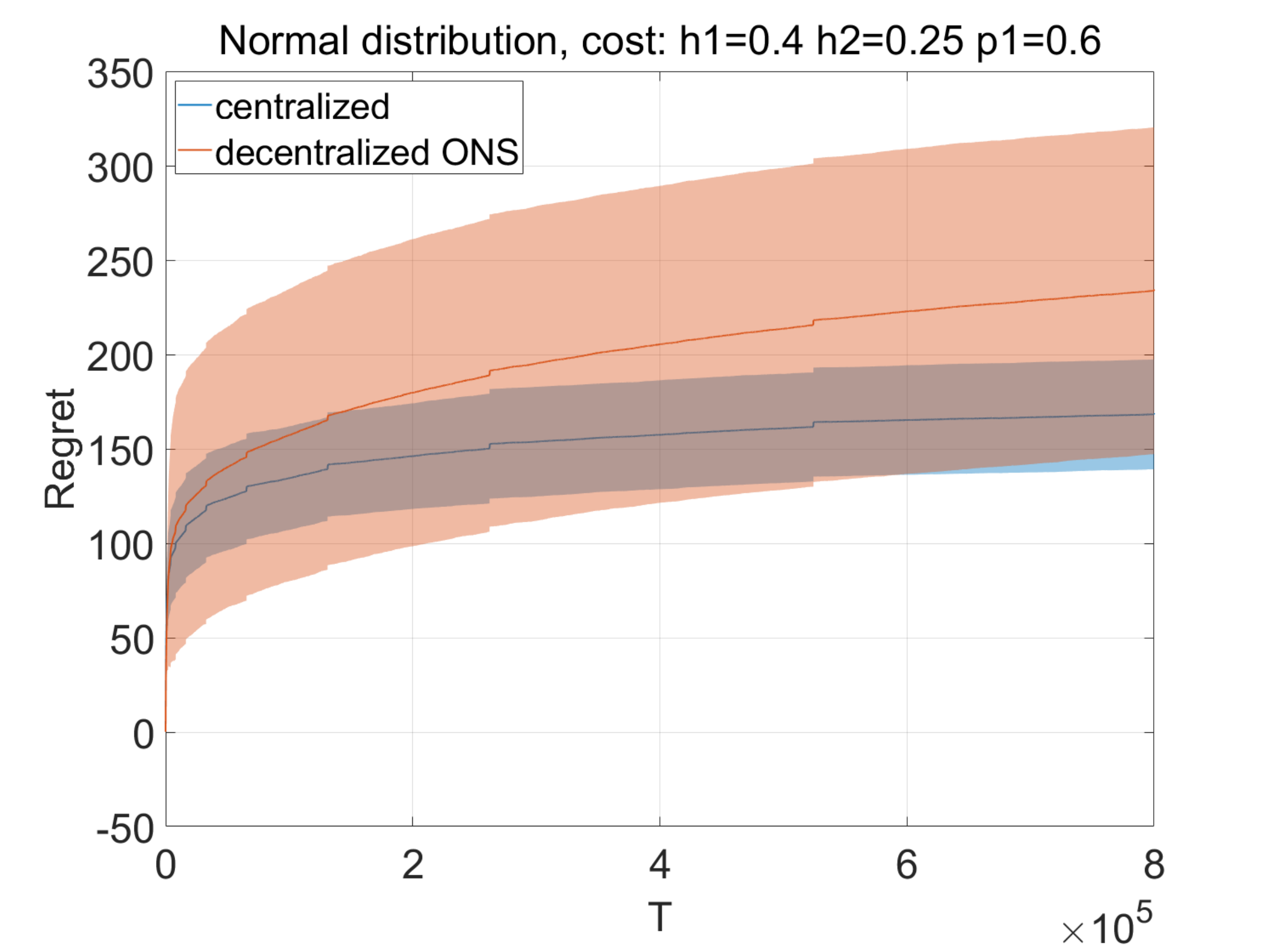}
\includegraphics[width=0.327\textwidth]{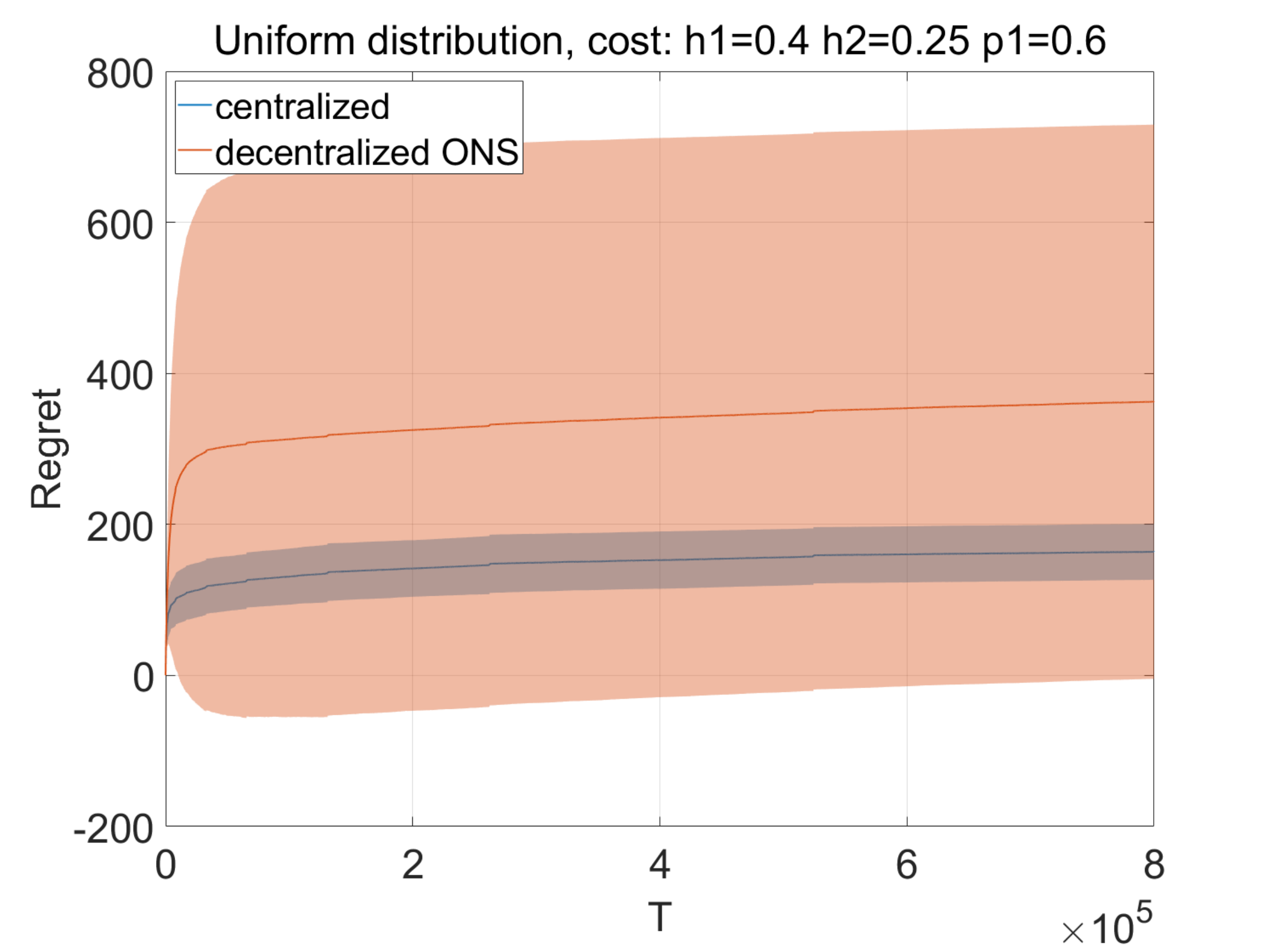}
\includegraphics[width=0.327\textwidth]{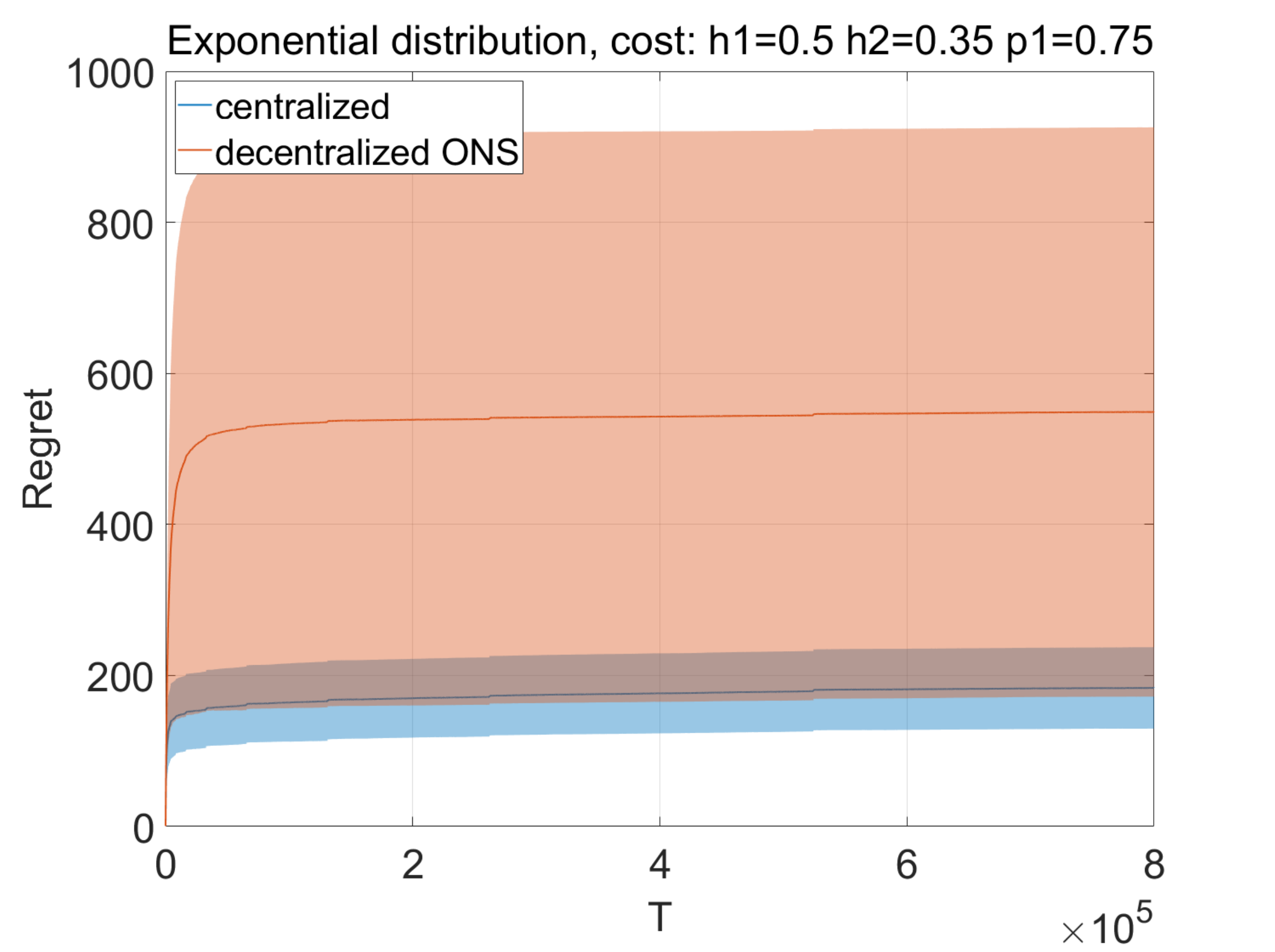}
\includegraphics[width=0.327\textwidth]{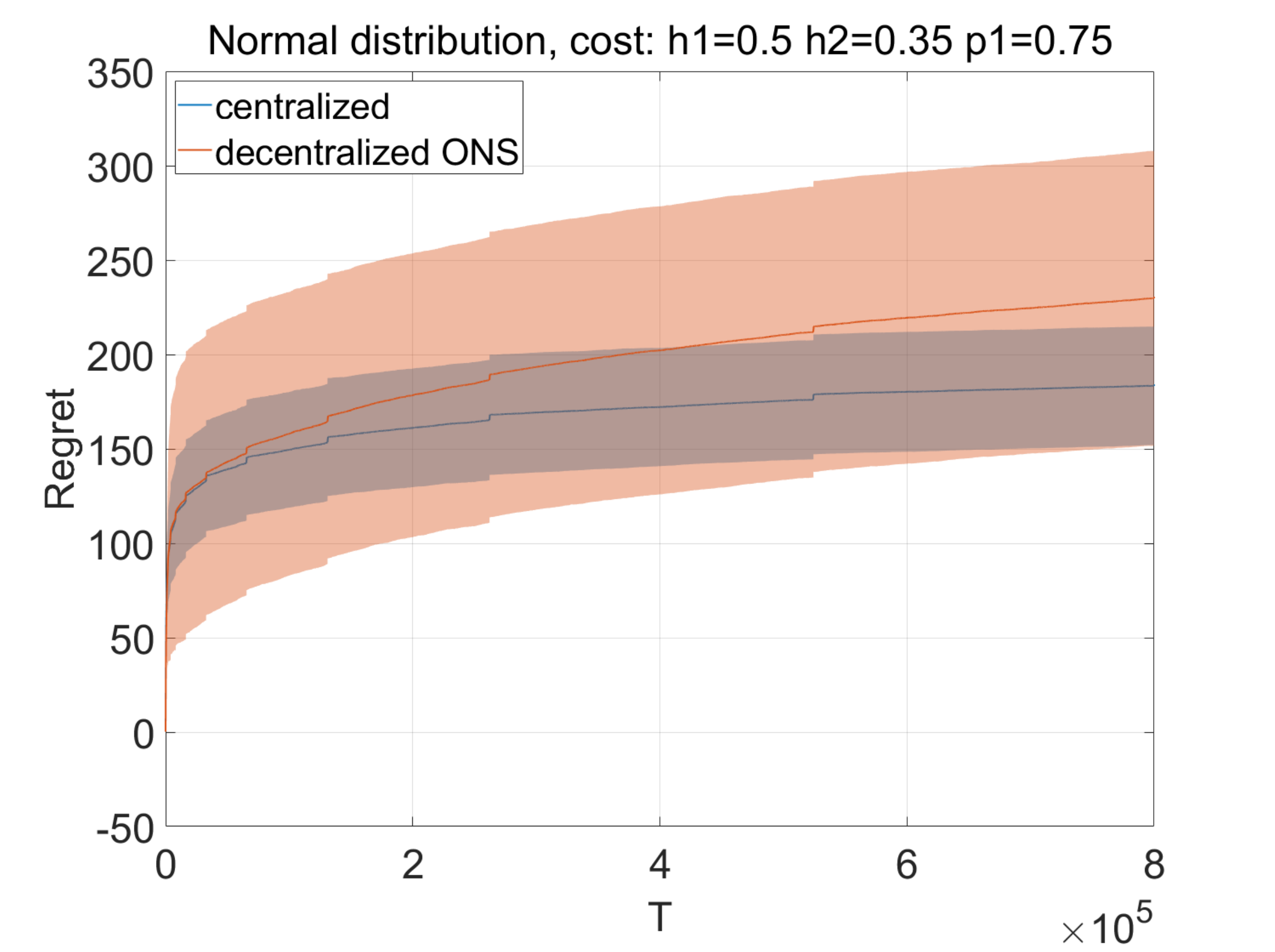}
\includegraphics[width=0.327\textwidth]{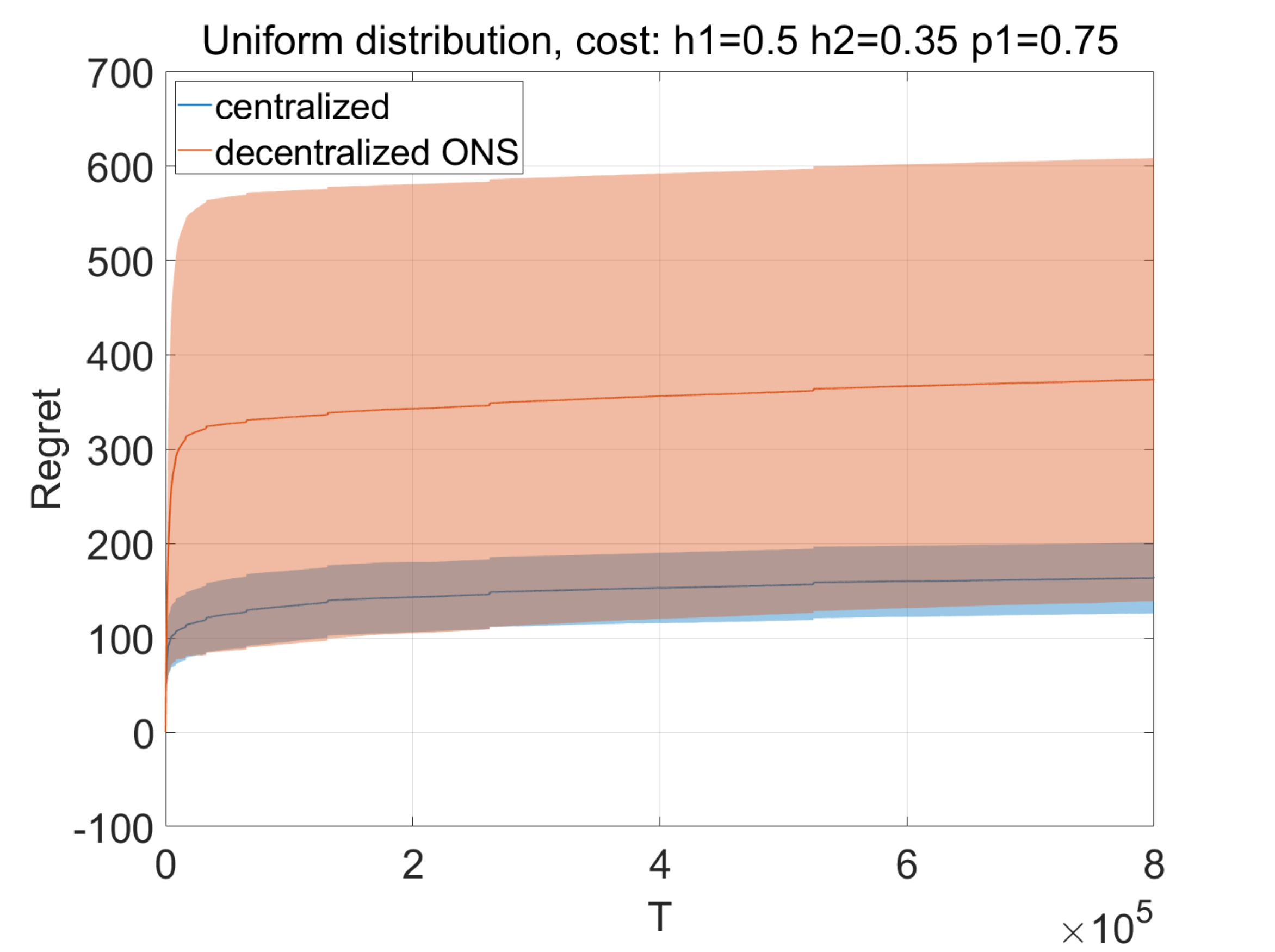}
\includegraphics[width=0.327\textwidth]{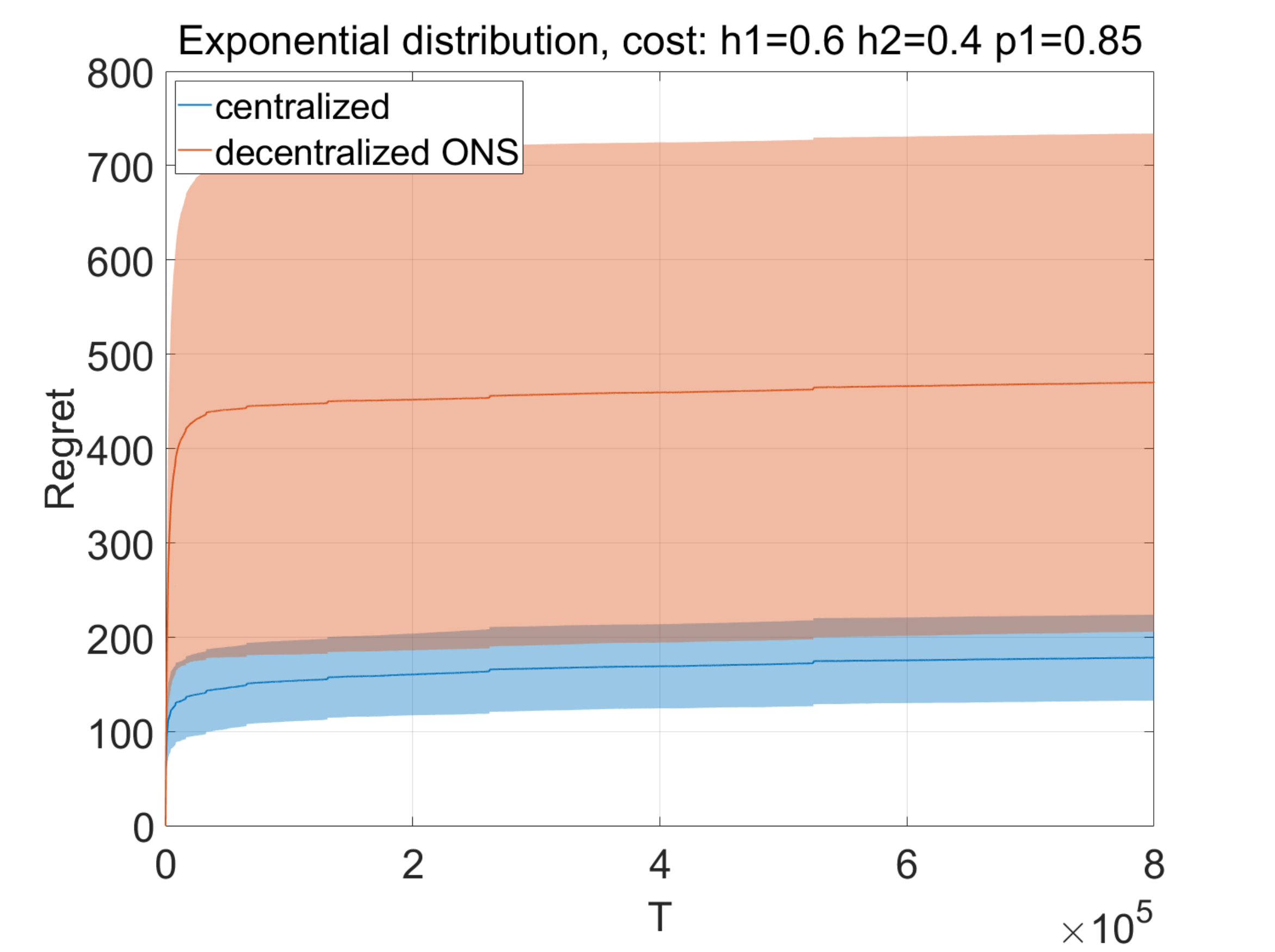}
\includegraphics[width=0.327\textwidth]{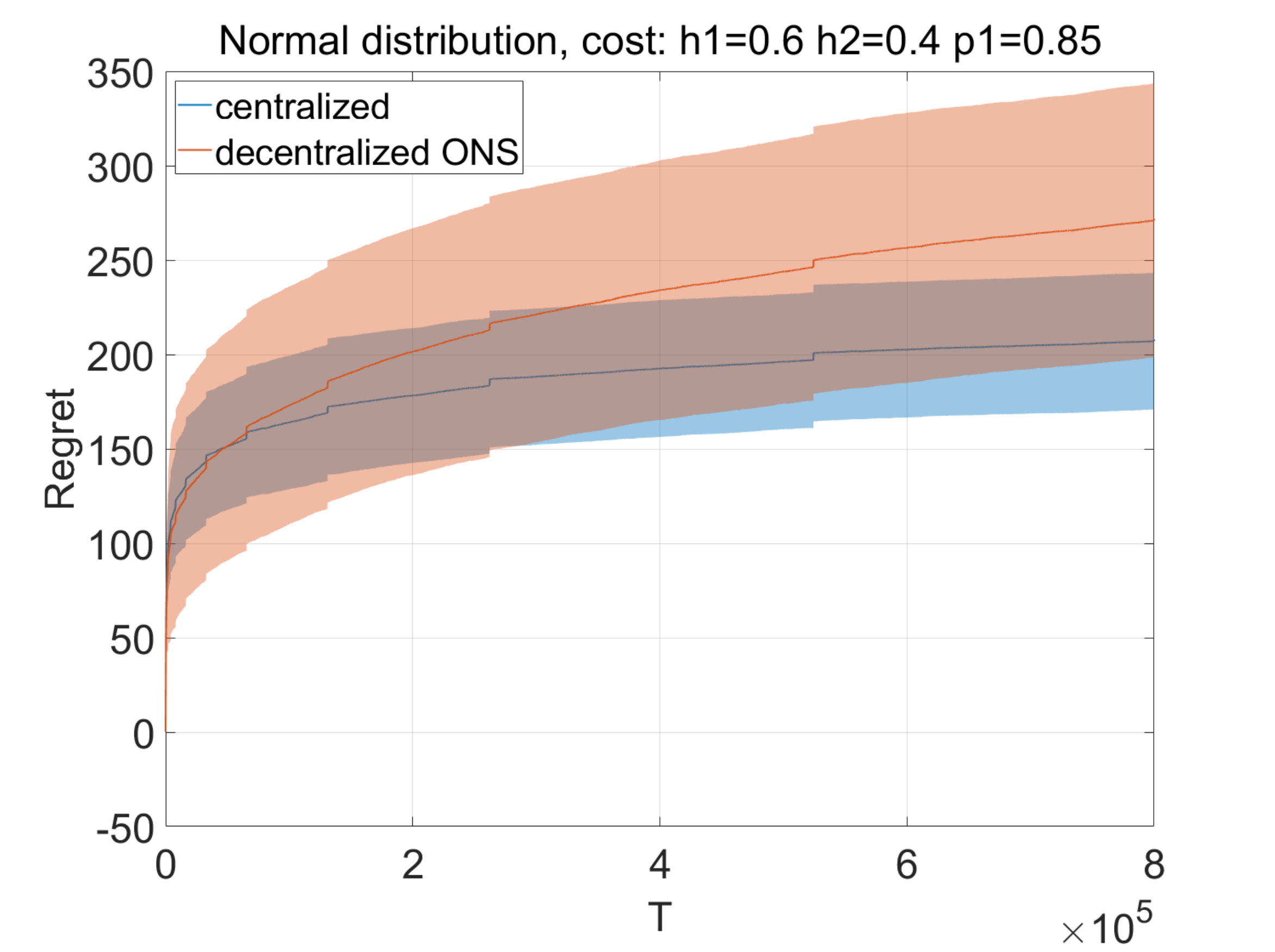}
\includegraphics[width=0.327\textwidth]{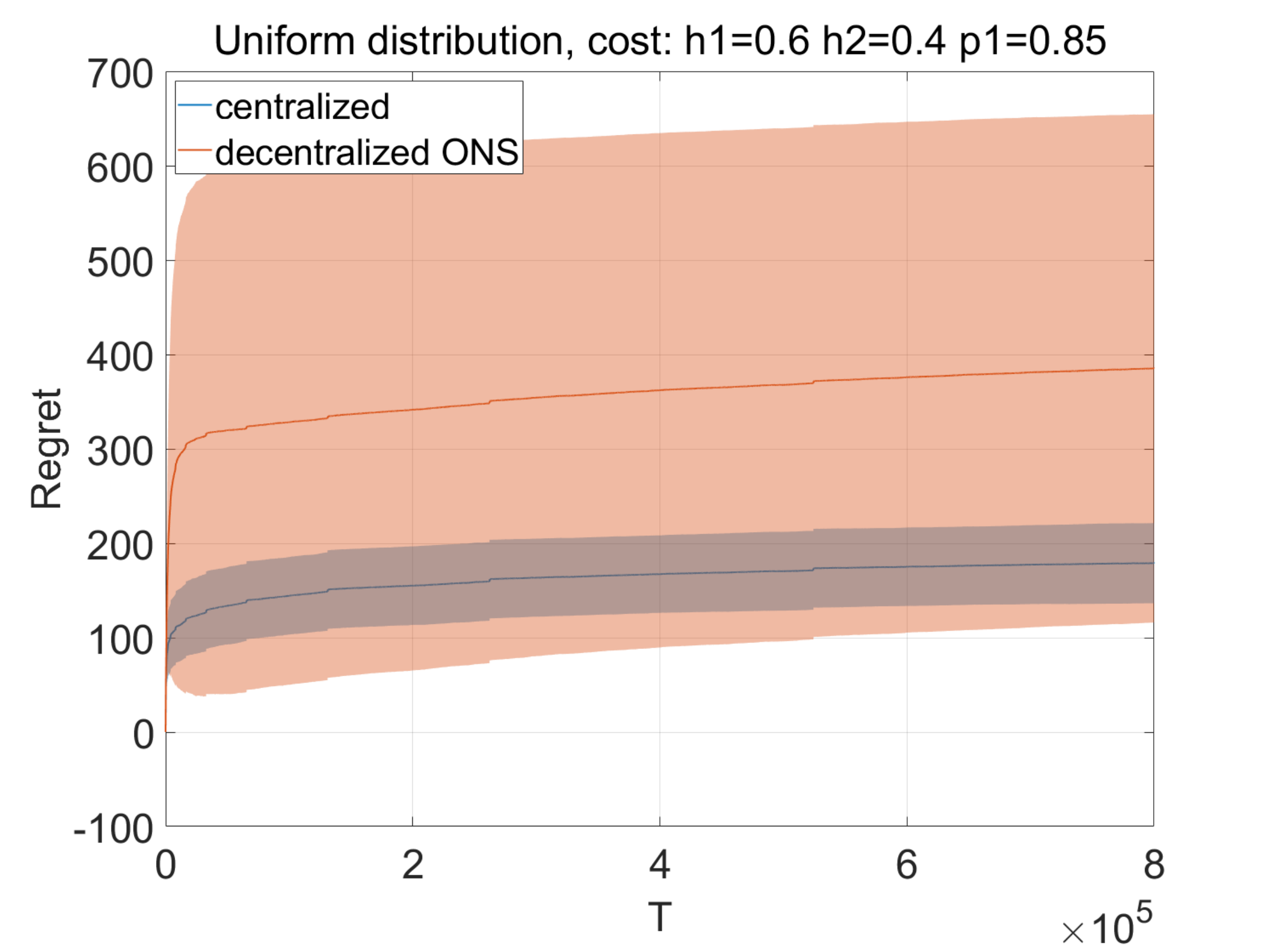}
\caption{
Empirical results of our algorithms applied to our model with cost parameters $(h_1,h_2,p_1)=\{(0.3,0.1,0.5),(0.4,0.25,0.6),(0.5,0.35,0.75),(0.6,0.4,0.85)\}$ and $T=800000$. Each column shows the results of a specific demand distribution with different cost parameter configurations. The algorithm is processed over $128$ trials of demand sequences drawn from the four distributions. The solid line is the mean over $128$ trials and the shaded area is mean $\pm$ std. The performance of of~\pref{alg:central-lazy-three} is shown in the blue curve (``centralized'') and the one of~\pref{alg:ons-lazy-three-non-decoupling} is shown in the orange curve (``decentralized ONS''). The results show the effectiveness of our proposed algorithms.}
\label{fig:coupling}
\end{figure}

\end{document}